\documentclass[acmsmall]{acmart}

\acmJournal{JACM}
\acmVolume{37} 
\acmNumber{4}
\acmArticle{1}
\acmMonth{8}

\usepackage{amsmath,amsfonts,amsthm}
\usepackage{mathtools}
\usepackage{booktabs}
\usepackage{hyperref}
\usepackage{algorithm, caption}
\usepackage[noend]{algpseudocode}

\usepackage{graphicx}
\usepackage{nicefrac}
\usepackage{diagbox}
\usepackage[misc]{ifsym}

\usepackage{enumitem}
\setlist[description]{leftmargin=0pt,labelindent=0pt}

\newcommand{\trace}{\text{tr}}
\newcommand{\mcl}{\mathcal{L}_G}

\newtheorem{theorem}{\bf{Theorem}}

\newtheorem{remark}[theorem]{\it{Remark}}
\newtheorem{problem}{\bf{Problem}}

\usepackage{caption}
\usepackage[list=true]{subcaption}
\usepackage{tabularx}
\newcolumntype{Y}{>{\centering\arraybackslash}X}

\newcolumntype{L}{>{\centering\arraybackslash\hsize=.5\hsize}X}
\newcolumntype{E}{>{\centering\arraybackslash\hsize=.3\hsize}X}

\newcolumntype{H}{>{\centering\arraybackslash}m{6cm}}
\newcolumntype{K}{>{\centering\arraybackslash}m{2cm}}
\newcolumntype{J}{>{\centering\arraybackslash}m{2cm}}
\newcolumntype{P}[1]{>{\centering\arraybackslash}p{#1}}

\date{Apr 5, 2023}

\usepackage{pgfplots}
\pgfplotsset{width=10cm,compat=1.9}
\usepackage{siunitx}
\usepackage{multirow}
\begin{document}

\setcopyright{acmcopyright}
\acmJournal{TKDD}
\acmYear{2023} \acmVolume{1} \acmNumber{1} \acmArticle{1} \acmMonth{1} \acmPrice{00.00}\acmDOI{00.0000/0000000}
%
%Constructing Graph Descriptors via Feature Extraction on Edge Streams

% \begin{frontmatter}

\title{Computing Graph Descriptors on Edge Streams}

\author{Zohair Raza Hassan}
\email{zh5337@rit.edu}
\affiliation{%
	\institution{Rochester Institute of Technology}
	\city{Rochester}
	\country{USA}
}

\author{Sarwan Ali}
\email{sali85@student.gsu.edu}
\affiliation{%
	\institution{Georgia State University}
	\city{Atlanta}
	\country{USA}
}

\author{Imdadullah Khan}
\email{imdad.khan@lums.edu.pk}
\affiliation{%
	\institution{Lahore University of Management Sciences}
	\city{Lahore}
	\country{Pakistan}
}

\author{Mudassir Shabbir}

\email{mudassir.shabbir@itu.edu.pk}
\affiliation{%
	\institution{Information Technology University}
	\city{Lahore}
	\country{Pakistan}
}

\author{Waseem Abbas}

\email{waseem.abbas@utdallas.edu}
\affiliation{%
	\institution{University of Texas at Dallas}
	\city{Dallas}
	\country{USA}
}

\renewcommand{\shortauthors}{Hassan, et al.}

\begin{abstract}

Feature extraction is an essential task in graph analytics. These feature vectors, called graph descriptors, are used in downstream vector-space-based graph analysis models. This idea has proved fruitful in the past, with spectral-based graph descriptors providing state-of-the-art classification accuracy. However, known algorithms to compute meaningful descriptors do not scale to large graphs since: (1) they require storing the entire graph in memory, and (2) the end-user has no control over the algorithm’s runtime. In this paper, we present streaming algorithms to approximately compute three different graph descriptors capturing the essential structure of graphs. Operating on edge streams allows us to avoid storing the entire graph in memory, and controlling the sample size enables us to keep the runtime of our algorithms within desired bounds. We demonstrate the efficacy of the proposed descriptors by analyzing the approximation error and classification accuracy. Our scalable algorithms compute descriptors of graphs with millions of edges within minutes. Moreover, these descriptors yield predictive accuracy comparable to the state-of-the-art methods but can be computed using only 25\% as much memory. 

\end{abstract}




\maketitle

\section{Introduction}
\label{sec:intro}

Graph analysis has a wide array of applications in various domains, from classifying chemicals based on their carcinogenicity~\cite{helma2001predictive} to determining the community structure in a friendship network~\cite{yanardag2015deep} and even detecting discontinuities within instant messaging interactions~\cite{berlingerio2013network}. The fundamental building block for analysis is a pairwise similarity (or distance) measure between graphs. However, efficient computation of such a measure is challenging: even the best-known solution for determining whether a pair of graphs are isomorphic has a quasi-polynomial runtime. Similarly, computing Graph Edit Distance~\cite{sanfeliu1983distance}, the minimum number of node/edge addition/deletions to interchange between two graphs is \textsc{NP-Hard}.

A relatively pragmatic approach is constructing fixed dimensional descriptors (vector embeddings) for graphs, allowing classical data mining algorithms that operate on vector spaces. Existing models using this approach can be categorized into (1) supervised models, which use deep learning methods to construct vector embeddings based on optimizing a given objective function~\cite{morris2019weisfeiler,xu2018powerful,Jin2021Toward} and (2) unsupervised models, which are based on graph-theoretic properties such as degree~\cite{sge,verma2017hunt}, the Laplacian eigenspectrum~\cite{kondor2016multiscale}, or the distribution of a fixed number of subgraphs~\cite{shervashidze2011weisfeiler,shervashidze2009efficient,ahmed2020interpret,shao2021motif,Mehri2021Mining,Duong2021Density}.

Unsupervised models construct general-purpose descriptors and do not require prior training on datasets. This approach has yielded great success; for example, descriptors based on spectral features (i.e., the graph's Laplacian) provide excellent results on benchmark graph classification datasets~\cite{tsitsulin2018netlsd,verma2017hunt}. The order (number of vertices) and size (number of edges) of the graph and the number and nature of features computed directly determine the runtime and memory costs of the methods. By computing more statistics, one can construct more expressive descriptors. However, this approach does not scale well to real-world graphs due to their growing magnitudes~\cite{koutra2016deltacon}.

Instead of storing and processing the entire graph, processing graphs as streams---one edge at a time---is a viable approach for limited memory settings~\cite{ahmed2013sample}. The features are approximated from a representative sample of fixed size. This approach of trading-off accuracy for time and space complexity has yielded promising results on various graph analysis tasks such as graphlet counting~\cite{Chen:2017:UFE:3110025.3110042}, butterfly counting~\cite{aida2022sgrapp,Sanei-Mehri:2019:FBE:3357384.3357983}, and triangle counting~\cite{shin2018tri,stefani2017triest}; despite storing a fraction of edges, these models have produced unbiased estimates with reasonably low error rates. Based on the success of these methods, our descriptors are designed to compute graph representations from edge streams, allowing us to compute features without storing the entire graph. In contrast, all existing descriptors and representation paradigms require storing the entire graph in memory.

This work is an extension of~\cite{hassan2020estimating}, wherein we proposed descriptors based on features obtained from graph streams. These descriptors are inspired by two existing works, the \textsc{Graphlet Kernel}~\cite{shervashidze2009efficient} and \textsc{NetSimile}~\cite{berlingerio2013network}, which compute local graph statistics as features. 
In this paper, we propose a new descriptor based on \textsc{NetLSD}~\cite{tsitsulin2018netlsd} along with the proofs and experiments showcasing the said descriptor's correctness and efficacy. We perform experiments on new benchmark datasets and provide data visualization of our proposed and \textsc{NetLSD} based embeddings using $t$-SNE.

\begin{figure}[!h]
\includegraphics[width=.95\linewidth]{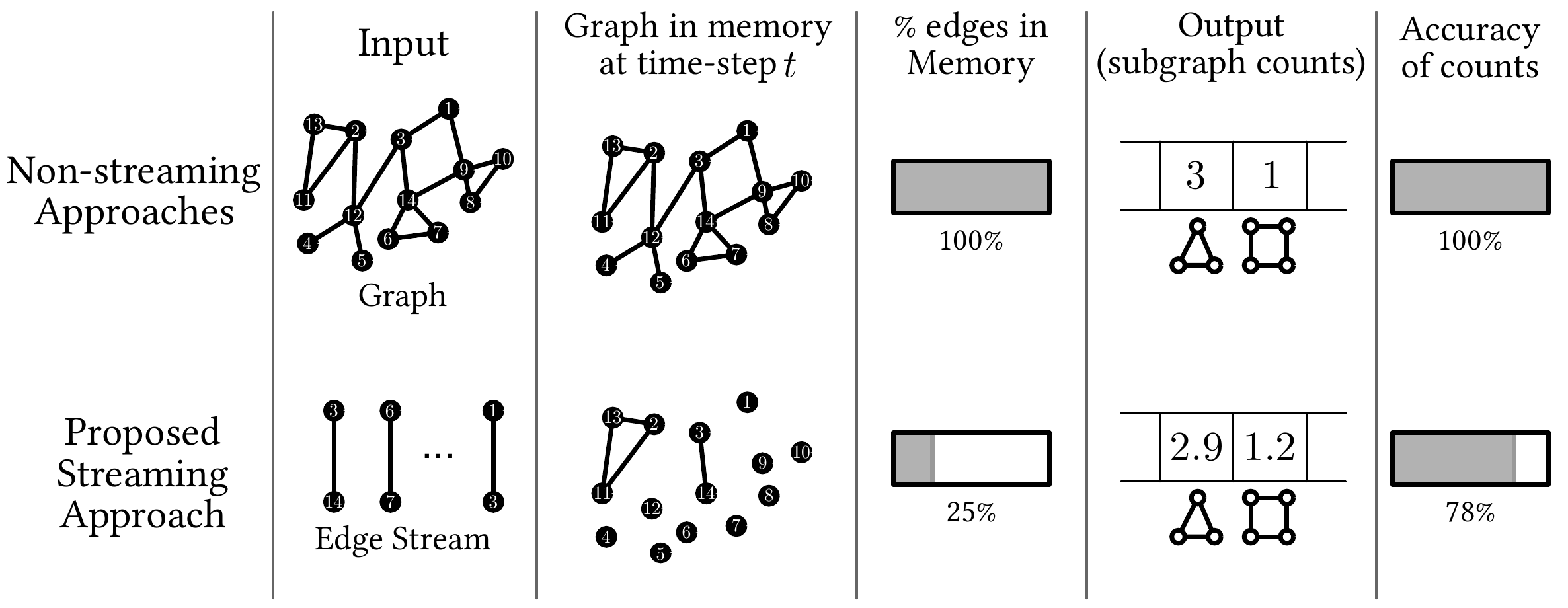}
\caption{This figure depicts the contrast between the typical approach for computing descriptors and our proposed approach. The descriptor in this example represents a graph by the counts of select subgraphs. Note how we tradeoff accuracy for memory consumption by keeping only a fraction of the graph in memory.}
\end{figure}

Our contributions are summarized as follows:
\begin{itemize}
\item We propose simple graph descriptors that run on edge streams.
\item We provide proofs to show how the features used in \textsc{NetSimile}~\cite{berlingerio2013network} and \textsc{NetLSD}~\cite{tsitsulin2018netlsd} can be computed using subgraph counts.
\item We restrict our algorithms' time and space complexity to scale linearly (for a fixed budget) in the order and size of the graph. We provide theoretical bounds on the time and space complexity of our algorithms.
\item Empirical evaluation on benchmark graph classification datasets demonstrates that our descriptors are comparable to other state-of-the-art descriptors with respect to classification accuracy. Moreover, our descriptors can scale to graphs with millions of nodes and edges because we do not require storing the entire graph in memory. 
\item We perform data visualization to show the (global) distribution of data points in the proposed and state-of-the-art (SOTA) descriptors. The visualization results show that the \textsc{santa} preserves the data distribution better \textsc{gabe} and \textsc{maeve} and is comparable to the SOTA descriptor, \textsc{NetLSD}.

\end{itemize}

The remaining paper is organized as follows. We review some of the related work in Section~\ref{sec:rw} and give a formal problem description in Section~\ref{sec:prelim}. We provide detail of our descriptors in Section~\ref{sec:sol}. Section~\ref{sec_experimental_evaluation} contains the experimental evaluation detail, including dataset statistics, preprocessing, hyperparameter values, and data visualization. In Section~\ref{sec:experiments} we report the experimental results of our method. Finally, we conclude the paper in Section~\ref{sec:conclusion}.

\section{Related Work} \label{sec:rw}

In this section, we review some closely related work on graph analysis. We discuss some distance/similar measures between graphs that are used in downstream machine learning algorithms. We also provide an overview of the basic paradigms for graph representation learning. 

\subsection{Pairwise Proximity Measure between Graphs}

A fundamental building block for analyzing large graphs is evaluating pairwise similarity/distance between graphs. The \textit{direct approach} to computing pairwise proximity considers the entire structure of both graphs. A simple and best-known distance measure between graphs is the {\em Graph Edit Distance} (\textsc{ged})~\cite{sanfeliu1983distance}. \textsc{ged}, like edit distance between sequences, counts the number of insertions, deletions, and substitutions of vertices and/or edges that are needed to transform one graph to the other. Runtimes of computing \textsc{ged} between two graphs are computationally prohibitive, restricting its applicability to graphs of very small orders and sizes. Another distance measure is based on permutations of vertices of one graph such that an error norm between the adjacency matrices of two graphs is minimum. Computing this distance and even relaxation of this distance is computationally expensive \cite{babai2016graph,bento2018family}. When there is a valid bijection between vertices of the two graphs, then a similar measure, \textsc{DeltaCon}~\cite{koutra2016deltacon}, yields excellent results. However, requiring a valid bijection limits the applicability of \textsc{DeltaCon} only to a collection of graphs on the same vertex set. 

\vskip.05in
The representation learning approach for graph analysis maps graphs into a vector space. Vector space machine learning algorithms are employed using a pairwise distance measure between the vector representations of graphs. We discuss three broad approaches in this vein.

\subsection{Kernel-Based Machine Learning Methods}

The \textit{kernel-based} machine learning methods represent each non-vector data item to a high dimensional vector. The feature vectors are based on counts (spectra) of all possible sub-structures of some fixed magnitude in the data item. A kernel function is then defined, usually as the dot-product of the pair of feature vectors. The pairwise kernel values between objects constitute a positive semi-definite matrix and serve as a similarity measure in the machine learning algorithm (e.g., SVM and kernel PCA). Explicit construction of feature vectors is computationally costly due to their large dimensionality. Therefore, in the so-called {\em kernel trick}, kernel values are directly evaluated based on objects. Kernel methods have yielded great successes for a variety of data such as images and sequences~\cite{Bo_ImageKernel,Kuksa_SequenceKernel,Farhan_SequenceKernel}. The most prominent graph kernels are the shortest-Path~\cite{borgwardt2005shortest}, Graphlet~\cite{shervashidze2009efficient}, the Weisfeller-Lehman~\cite{shervashidze2011weisfeiler}, and the hierarchical~\cite{kondor2016multiscale} kernels. The computational and space complexity of the kernel matrix make kernel-based methods infeasible for large datasets of massive graphs.

\subsection{Deep Learning Based Methods} 

The deep learning approach to representation learning is to train a \textit{neural network} for embedding objects into Euclidean space. The goal here is to map `similar' objects to `close-by' points in $\mathbb{R}^d$. Deep learning-based methods and domain-specific techniques have been successfully used for embedding nodes in networks~\cite{duran2017learning,grover2016node2vec,cao2015grarep} and graphs~\cite{xu2018powerful,morris2019weisfeiler,Yang2020jane,yang20183}. Vector-space-based machine learning methods are then employed on these embeddings for data analysis. However, these approaches are data-hungry and computationally prohibitive~\cite{Shakeel2020Multi}, hindering their scalability to graphs of large orders and sizes.

\subsection{Descriptor Computation Methods}

The \textit{descriptor} learning paradigm differs from kernel methods in that the dimensionality of the feature vectors is much smaller than the kernel-based features. Unlike neural network-based models, the features are explainable and hand-picked using domain-specific knowledge~\cite{berlingerio2013network,Ali202Attribute}. 
One such graph descriptor, \textsc{NetSimile}~\cite{berlingerio2013network}, represents a graph by a vector of aggregates of various vertex-level features. It considers seven features for each vertex, such as degree, clustering coefficient, and parameters of vertices' neighbors and their ``ego-networks,'' and applies the aggregator functions, such as median, mean, standard deviation, skewness, and kurtosis, across each feature. Stochastic Graphlet Embedding~\cite{sge} proposes a graph descriptor based on random walks over graphs to extract graphlets (sub-structures) of increasing order. Similar to this sub-structural approach is the Higher Order Structure Descriptor~\cite{ahmed2020interpret}, which iteratively compresses graphlets within a graph to generate ``higher-order'' graphs and constructs histograms of the graphlet counts in each graph. More recently, \textsc{feather} was introduced as a descriptor that computes node-level feature vectors using a complex characteristic function and aggregates these to construct graph embeddings~\cite{rozemberczki2020feather}. There has been a trend towards using graph spectra \cite{AHMAD2020Combinatorial,Ahmad2016AusDM,Tariq2017Scalable} to learn descriptors \cite{verma2017hunt,tsitsulin2018netlsd}.
These descriptors are relatively computationally expensive but have excellent classification performance.
An exact method, Von Neumann Graph Entropy (VNGE) is proposed in~\cite{braunstein2006laplacian,chen2019fast} for graph comparison. Being an exact method, VNGE does not scale to large graphs. An approximate solution of NetLSD and VNGE, called SLaQ~\cite{tsitsulin2020just}, computes spectral distances between graphs with multi-billion nodes and edges. Although computationally efficient, SLaQ keeps the entire graph in the memory during the processing, making it costly in terms of space efficiency.

Most of the above approaches require multiple passes over the entire input graph. The resulting space complexity renders them applicable only to graphs of small orders and sizes. On the other hand, real-world graphs are dynamic and enormous in their magnitude. Algorithms that perform a single pass over the input stream and have low memory requirements~\cite{Ali2019Detecting} are best suited for modern-day graphs. An algorithm that computes the output with provable approximation guarantees is sufficient for the single-pass and sub-linear memory requirements. There have been few recent algorithms for counting specific substructures in a streamed graph owing to the inherent difficulty of the streaming model. These include approximately computing the number of triangles~\cite{stefani2017triest} in graphs, induced subgraphs of order three and four~\cite{Chen:2017:UFE:3110025.3110042} in graphs, and cycles of length four in bipartite graphs~\cite{Sanei-Mehri:2019:FBE:3357384.3357983}. 

\section{Preliminaries}
\label{sec:prelim}

In this section, we discuss the necessary prerequisites required to follow our work. We describe the graph nomenclature, followed by the description of graph descriptors, streams, and constraints imposed on our algorithm.

\subsection{Graph Nomenclature}
\label{sec:note}

This section gives relevant notation and terminology for the rest of the paper, followed by a precise formulation of our main problem. A description of the notations for common terms used in this paper has been provided in Table~\ref{tab:not:overa_1}. Notation tables specific to each descriptor have been provided in their sections. 

Let $G = \left(V_G, E_G\right)$ be an undirected graph, where $V_G$ is the set of vertices and $E_G$ is the set of edges. We denote vertices of $G$ by integers in the range $[0,|V_G|-1]$. We refer to $|V_G|$ and $|E_G|$ as the order and size of $G$, respectively. In this paper, we only consider simple graphs (i.e., graphs with no self-loops and multi-edges) and unweighted graphs. For a vertex  $v \in V_G$, we denote by $N_G\left(v\right)$, the set of neighbors of $v$, i.e., the set of vertices that are adjacent to $v$. More formally, $N_G\left(v\right) = \{ u : \left(u, v\right) \in E_G \}$. The degree of a vertex $v$ is denoted by $d^v_G$, i.e.,  $d^v_G:= |N_G\left(v\right)|$. A pair of vertices $u,v\in V_G$ are said to be connected if there is a path between $u$ and $v$, i.e., there exists a sequence of vertices $u=v_1,\ldots,v_k = v$, where for $1\leq i \leq k-1$, $(v_i,v_{i+1})\in E_G$. The length of a path is the number of vertices in it. A graph is called connected iff all pairs of vertices in $V_G$ are connected.

A graph $G' = \left(V_{G'}, E_{G'}\right)$ is called a subgraph of $G = (V_G,E_G)$ if $V_{G'} \subseteq V_G$ and $E_{G'} \subset E_G$ such that edges in $E_{G'}$ are incident only on the vertices present in $V_{G'}$, i.e., $E_{G'} \subseteq \{ \left(u, v\right) : \left(u, v\right) \in E_G \land u, v \in V_{G'}\}$. If all edges incident on vertices in $V_{G'}$ are in $E_{G'}$ ($E_{G'} = \{ \left(u, v\right) : \left(u, v\right) \in E_G \land u, v \in V_{G'}\}$), then $G'$ is called an induced subgraph of $G$.

When vertices of a graph $G_1$ can be relabelled in such a way that we get another graph $G_2$, then we say that $G_1$ and $G_2$ are isomorphic. In other words, $G_1$ and $G_2$ are isomorphic iff there exists a bijection $\pi: V_{G_1} \rightarrow V_{G_2}$ such that $E_{G_2} = \{ (\pi(u), \pi(v)) : (u,v) \in E_{G_1} \}$. For a graph $F = \left(V_{F}, E_{F}\right)$, let $H^{F}_G$ (resp., $\widehat{H}^{F}_G$) be the set of subgraphs (resp., induced subgraphs) of $G$ that are isomorphic to $F$.

\begin{table}[h!]
\centering
\begin{tabular}{cl}
\toprule
\textbf{Notation}  & \textbf{Description}\\ \midrule
$G, F, G'$ & Common terms for graphs \\
$V_G, E_G$ & Vertex and edge set for a graph $G$ \\[.02in]
$N_G(v), d^v_G$ & Neighborhood and degree of a vertex $v \in V_G$ \\[.02in]
$S, e_t$ & A stream of edges, and the edge arriving at time-step $t$ \\[.04in]
$H^F_G$ & Set of subgraphs of $G$ isomorphic to $F$ \\[.04in]
$\widehat{H}^F_G$  & Set of induced subgraphs of $G$ isomorphic to $F$ \\[.04in]
$N^F_G$ & Estimate of $|H^F_G|$ \\[.02in]
$p^F_t$ & Probability of detecting $F$ in $\widetilde{E_G}$, at the $t^{th}$ time-step \\ \bottomrule
\end{tabular}
\caption{Notation table for common terms used throughout the paper.}
\label{tab:not:overa_1}
\end{table}

\subsection{Graph Descriptors and Streams}\label{sec_graph_descr_and_stream}

A graph descriptor is a mapping from the family of all possible graphs (undirected, unweighted and simple, in our case) to a set of $d$-dimensional real vectors. More formally, let $\mathcal{G}$ be the set of all possible graphs. A descriptor $\varphi$ is a function, $\varphi: \mathcal{G} \rightarrow \mathbb{R}^d$. The primary motivation for using descriptors for graph analysis is to map graphs (possibly of varying sizes and orders) into a fixed-dimensional vector space, independent of the representation of graphs~\cite{berlingerio2013network,tsitsulin2018netlsd}. A direct comparison of the number of certain subgraphs in two graphs of different orders and/or sizes is not very meaningful, as larger graphs will naturally have more subgraphs. Moreover, descriptors enable the application of vector-space-based machine learning algorithms for graph analysis tasks, often using the $\ell_2$-distance (Euclidean distance) as the proximity measure. Our descriptors are graph-theoretic and apply to graphs of varying magnitudes.

Let $S = e_1, e_2, \ldots, e_{|E_G|}$ be a sequence of edges in a fixed order, i.e., $e_t = \left(u_t, v_t\right)$ is the $t^{th}$ edge. We assume an online setting wherein the input graph is modeled as a stream of edges, i.e., we assume that elements of $S$ are input to the algorithm one at a time. The following constraints are imposed on our algorithms: 

\begin{itemize}
    \item[C1]{\bf Constant Number of Passes:} The algorithm must do processing in a constant number of passes over the graph stream. Our algorithms require two passes at most.
    \item[C2]{\bf Limited Space:} The algorithm can store at most $b$ edges during the execution. We refer to $b$ as the budget and $\widetilde{E_G}$ as the sample.
    \item[C3]{\bf Linear Complexity:} The time and space complexity of the algorithms must be linear in the order and size of the graph, with fixed $b$.
\end{itemize}

\subsection{Estimating Connected Subgraph Counts on Edge Streams}
\label{sec:est}

In this section, we formally define the subgraph estimation problem within our constraints and describe the solution to this problem used throughout our proposed descriptors.

\begin{problem}[Connected Subgraph Estimation on Edge Streams]
\label{prob:subest}
Let $S$ be a stream of edges, $ e_1, e_2, \ldots, e_{|E_G|}$ for some graph $G = (V_G, E_G)$. Let $F = (V_F, E_F)$ be a connected graph such that $|V_F| \ll |V_G|$ (i.e. $F$ is significantly smaller than $G$). Compute an estimate, $N^F_G$, of $|H^F_G|$ while storing at most $b$ edges at any given instant.
\end{problem}

The basic strategy for solving Problem~\ref{prob:subest} involves two things: (1) an algorithm that counts the number of instances of a subgraph $F$ that an edge belongs to, and (2) a sampling scheme that allows us to compute the probability of detecting an instance of $F$ in our sample, denoted by $p^F_t$, at the arrival of the $t^{th}$ edge~\cite{Chen:2017:UFE:3110025.3110042,shin2018tri,stefani2017triest}.
The basic streaming algorithm maintains a representative sample of edges from the stream, and for each next edge $e_t$, it estimates the number of subgraphs in the sample containing the edge $e_t$. This estimate is scaled according to the sample size. 
At the arrival of $e_t$, the estimate of $|H^F_G|$ is incremented by $1/p^F_t$ for all instances of $F$ that $e_t$ belongs to in our sample $\widetilde{E_G} \cup \{e_t\}$. A pseudo-code is provided in Algorithm~\ref{alg:pseudo}. This approach computes estimates equal to $|H^F_G|$ on expectation.

\begin{theorem}
\label{thm:unbiased}
Algorithm~\ref{alg:pseudo} provides unbiased estimates: $\mathbb{E}[N^F_G] = |H^F_G|$.
\end{theorem}

\begin{proof}
Let $h$ be a subgraph in $H^F_G$. We define $X_h$ as a random variable such that $X_h = 1/p^F_t$ if $h$ is detected at the arrival of $e_t$, and 0 otherwise. 
Clearly, $N^F_G = \sum_{h \in H^F_G} X_h$, and $\mathbb{E}[X_h] = (1/p^F_t)\times p^F_t = 1$. Thus,

\begin{equation*}
    \mathbb{E}\left[N^F_G\right] = \mathbb{E}\bigg[\sum_{h \in H^F_G} X_h\bigg] = \sum_{h \in H^F_G} \mathbb{E}\left[X_h\right] = \sum_{h \in H^F_G} 1 = \left|H^F_G\right|
\end{equation*}

\end{proof}

When $e_t$ arrives, the only subgraphs counted are the ones that $e_t$ belongs to. This ensures that no subgraph is counted more than once. Due to its previous success in subgraph estimation~\cite{Chen:2017:UFE:3110025.3110042,shin2018tri,stefani2017triest}, we utilize reservoir sampling~\cite{Vitter:1985:RSR:3147.3165}. 
With reservoir sampling, the probability of detecting a subgraph $F$ at the arrival of $e_t$ is equal to the probability that $F$'s other $|E_F| -1$ edges are present in the sample after $t-1$ time-steps. Thus, we can write:

\begin{equation*}
p^F_t = \min \bigg\{ 1, \prod^{|E_F| -2}_{i = 0} \frac{b - i}{t - 1 -i} \bigg\}
\end{equation*}

\begin{algorithm}[ht]
\caption{Compute-Estimate($S,F,b$)}
\label{alg:pseudo}
\begin{algorithmic}[1]

\State $\widetilde{E_G} \gets \emptyset$ 
\State $N^F_G \gets 0$ 
\For{$t=1$ to $|E_G|$ } 
 \State $G' \gets (V_G, \widetilde{E_G} \cup \{e_t\})$
 \State $N \gets $ number of instances of $F$ in $G'$ that $e_t$ belongs to. 
 \State $N^F_G \gets N^F_G  + N \times \nicefrac{1}{p^F_t}$
 \State Discard or store $e_t$ in $\widetilde{E_G}$, based on the sampling method and $b$
 \EndFor
\State \Return $N^F_G$
\end{algorithmic}
\end{algorithm}

To analyze the effect of the budget on our estimates, we derive an upper bound for the variance of $N_G^F$. Although loose, the bound shows that better estimates are obtained for any connected graph $F$ with increasing $b$.

\begin{theorem}
\label{thm:var}
Let $N^F_G$ be the estimate of $|H^F_G|$ obtained using Algorithm~\ref{alg:pseudo} with reservoir sampling. Then,
\begin{equation*}
    \mathrm{Var}[N^F_G] \leq |H^F_G|^2 \prod^{|E_F| -2}_{i = 0} \dfrac{|E_G| - i}{b -i}
\end{equation*}
\end{theorem}
\begin{proof}
The theorem is true when $b \geq |E_G| - 1$. Thus, we focus on the case when $b < |E_G| - 1$. 
As in Theorem~\ref{thm:unbiased}, we define $X_h$ as a random variable such that $X_h = 1/p^F_t$ if $h$ is detected at the arrival of $e_t$, and 0 otherwise. 
It is clear from the definition of $p^F_t$ that $ p^F_t \geq p^F_{t+1} $ for all $t$, and thus $p^F_t \geq p^F_{|E_G|}$. Hence, $\mathrm{Var}[X_h] = \mathbb{E}[X_h^2] - \mathbb{E}[X_h]^2 = 1/p^F_t - 1 \leq 1/p^F_{|E_G|}$. 
The Cauchy-Schwarz inequality can be used to bound the total variance like so:
\begin{align*}
    \mathrm{Var}[N^F_G] &= \sum_{h \in H^F_G} \sum_{h' \in H^F_G} \mathrm{Cov}[X_h,X_{h'}] \leq \sum_{h \in H^F_G} \sum_{h' \in H^F_G} \sqrt{\mathrm{Var}[X_h] \mathrm{Var}[X_{h'}]} \\
    &\leq \sum_{h \in H^F_G} \sum_{h' \in H^F_G} \frac{1}{p^F_{|E_G|}} = |H^F_G|^2 \prod^{|E_F| -2}_{i = 0} \frac{|E_G| - 1 - i}{b -i}
\end{align*}
\end{proof}

As observed in~\cite{shin2018tri}, we note that this methodology can be used to estimate vertex counts (the number of subgraphs that each vertex belongs to) as well. Moreover, Theorems~\ref{thm:unbiased} and \ref{thm:var} can also be extended to vertex counts.

\subsection{Improving Estimation Quality with Multiple Workers}

Shin et al. proposed a model for triangle estimation which takes advantage of a master machine and multiple worker machines that work in parallel. Each machine independently receives edge streams, estimates triangle counts, then sends them to the master machine, which aggregates each machine's estimate~\cite{shin2018tri}.
They show that using $W$ worker machines decreases the estimates' variance by a factor of $1/W$. Thus, we use their approach to improve the quality of subgraph estimations used in our descriptors. 

\section{Graph Descriptors}\label{sec:sol}

In this section, we describe three graph descriptors:
Graphlet Amounts via Budgeted Estimates (\textsc{gabe}), Moments of Attributes Estimated on Vertices Efficiently (\textsc{maeve}), and Spectral Attributes for Networks via Taylor Approximation (\textsc{santa}). These are based on the \textsc{Graphlet Kernel}~\cite{shervashidze2009efficient}, \textsc{NetSimile}~\cite{berlingerio2013network}, and \textsc{NetLSD}~\cite{tsitsulin2018netlsd}, respectively. For each descriptor, we describe its features and how it can be computed using subgraph enumeration. We also analyze their algorithms to show that constraints $C1$, $C2$, and $C3$ (from Section~\ref{sec_graph_descr_and_stream}) are met. Each descriptor's details have been summarized in Table~\ref{tab:summary}.

\begin{table}[h!]
\centering
\begin{tabular}{@{}lcccc@{}}
\toprule
\textbf{Name}           & \textbf{Summarized Description}      & \textbf{\# Passes} & \textbf{Time Complexity}     & \textbf{Space Complexity} \\ \midrule
\textsc{gabe}  & Normalized subgraph counts & 1           & $O(b\log b|E_G|)$   & $O(b + |V_G|)$   \\
\textsc{maeve} & Aggregated local features   & 1           & $O(b|E_G| + |V_G|)$ & $O(b + |V_G|)$   \\
\textsc{santa} & Functions on eigenspectrum & 2 & $O(b\log b|E_G|)$ & $O(b + |V_G|)$ \\ \bottomrule
\end{tabular}%

\caption{A summary of the proposed descriptors.} 
\label{tab:summary}
\end{table}

\subsection{GABE: Graphlet Amounts via Budgeted Estimates}
\label{sec:gabe}

The first descriptor we propose is based on normalized subgraph counts. Subgraph counts have been popular in graph classification literature (e.g.,~\cite{shervashidze2009efficient,sge,ahmed2020interpret,shao2021motif}) and have been shown to provide fruitful descriptors by capturing the prevalence of small local structures throughout a graph. 

Let $\mathcal{F}_k$ be the set of graphs with order $k$. 
In their work on the \textsc{Graphlet Kernel}, Shervashidze et al.~\cite{shervashidze2009efficient} propose measuring the similarity between two graphs $G_1$ and $G_2$ by counting the number of graphlets in $\mathcal{F}_k$ and computing the inner product $\langle \phi_k(G_1), \phi_k(G_2) \rangle$, where for a given $k$ and graphs $F_i \in \mathcal{F}_k$:
\begin{equation*}
\phi_k(G) := \frac{1}{{\binom{|V_G|}{k}}} \begin{bmatrix}
\left|\widehat{H}^{F_1}_G\right| & \cdots & 
\left|\widehat{H}^{F_{|\mathcal{F}_k|-1}}_G\right| & 
\left|\widehat{H}^{F_{|\mathcal{F}_k|}}_G\right|
\end{bmatrix}^\intercal
\end{equation*}

\begin{table}[h!]
\centering
\begin{tabular}{cl}
\toprule
\textbf{Notation} & \textbf{Description} \\ \midrule
$k$ & Maximum order of a subgraph enumerated in $G$ by \textsc{gabe} \\[.02in]
$\mathcal{F}$ & Family of all graphs with at most four vertices \\[.02in]
$\mathcal{O}$ & Overlap matrix \\[.04in]
$\mathcal{H^F_G}$ & Vector of subgraph counts \\[.07in]
$\widehat{\mathcal{H^F_G}}$ & Vector of induced subgraph counts \\[.02in] \bottomrule
\end{tabular}
\caption{Notation table for Section~\ref{sec:gabe}.}
\label{tab:not:gabe}
\end{table}

They compute the exact counts of all graphlets in $\mathcal{F}_k$ for $k \in \{3,4,5\}$. Unfortunately, their algorithm uses adjacency matrices and adjacency lists, which take $O(|V_G|^2)$ and $O(|V_G| + |E_G|)$ space, respectively. Moreover, the time complexity is $O(|V_G|d^{k-1})$, where $d = \max_{v \in V_G} d^v_G$ is the maximum degree across all vertices in $G$. Thus, their algorithm does not scale well to large graphs. Although the authors introduce a sampling method to approximate $\phi_k(G)$, it requires storing the entire graph in memory and therefore does not meet our constraints.

We construct our descriptors by estimating subgraph counts and using linear combinations of these counts to compute induced subgraph counts, similar to the methodology used by Chen et al.~\cite{Chen:2017:UFE:3110025.3110042}.
The linear combinations are based on the overlap of graphs of the same order. 
Using this approach, we estimate, for a given graph $G$, $\phi_k(G)$ for $k \in \{2,3,4\}$. Each $\phi_k(G)$ is concatenated to construct our final descriptor. There are 17 graphs with $\leq 4$ vertices, each shown in Figure~\ref{fig:gabe1}. Note that Chen et al. do not discuss the estimation of disconnected subgraphs. We discuss how we compute these in the section to follow.

\subsubsection{Induced Subgraph Counts} 

Let $\mathcal{F} = \{F_1, F_2, \ldots, F_{17}\}$ be the set of all graphs with at most four vertices. Let $\mathcal{H}^{\mathcal{F}}_G$ (resp., $\widehat{\mathcal{H}}^{\mathcal{F}}_G$) be a $|\mathcal{F}|$-dimensional vector where the $i^{th}$ entry corresponds to $|H^{F_i}_G|$ (resp., $|\widehat{H}^{F_i}_G|$).
Let $\mathcal{O}$ be an ``overlap matrix.'' $\mathcal{O}$ is an $|\mathcal{F}| \times |\mathcal{F}|$ matrix such that the element $O(i,j)$ corresponds to the number of subgraphs of $F_j$ isomorphic to $F_i$ when $F_i$ and $F_j$ have the same number of vertices. The value is set to zero when the orders $|V_{F_i}|$ and $|V_{F_j}|$ are not equal.

Observe that $\mathcal{H}^{\mathcal{F}}_G = \mathcal{O}\widehat{\mathcal{H}}^{\mathcal{F}}_G$. 
This is because for a single subgraph $F_i \in \mathcal{F}$, the overlap matrix counts the number of $F_i$'s induced in $G$, and the number of $F_i$'s that occur in induced $F_j$'s for each $F_j \in \mathcal{F}$ such that $F_i$ is a subgraph of $F_j$. Note that $\mathcal{O}$ is invertible since it is an upper triangular matrix. Thus we can compute the vector of induced subgraphs using the formula $\widehat{\mathcal{H}}^{\mathcal{F}}_G = \mathcal{O}^{-1}\mathcal{H}^{\mathcal{F}}_G$. 
Thus, our proposed approach is to compute $N^{F_i}_G$ using our estimation technique, and $\widehat{N}^{F_i}_G$ using the overlap matrix, where
$N^{F_i}_G$ (resp., $\widehat{N}^{F_i}_G$) is the estimate of $|H^{F_i}_G|$ (resp., $|\widehat{H}^{F_i}_G|$).

The estimated counts of each subgraph are computed as follows:
\begin{enumerate}
    \item \textit{Connected Subgraphs.} The graphs $F_6, F_{13}, \ldots, F_{17}$ are computed as described in Section~\ref{sec:est}; edge-centric algorithms were written to enumerate over all instances in the sample $\widetilde{E_G} \cup \{e_t\}$ and increment the estimates as described earlier. The counts of the star graphs, $F_5$ and $F_12$, are computed using the degrees of each vertex and the formulas written in Table~\ref{tab:gabeform}. $F_2$ is simply equal to the number of edges in $G$. 
    \item \textit{Disconnected Subgraphs.} Combinatorial formulas based on the estimates of connected subgraphs, $|E_G|$, and $|V_G|$. Note that the size of $G$ can be computed by keeping track of the number of edges received, and the order can be computed by tracking the maximum vertex label, on account of each vertex being labeled in the range $[0,|V_G|-1]$.
\end{enumerate}

\begin{figure}[h!]
    \centering
    \includegraphics[width=\linewidth]{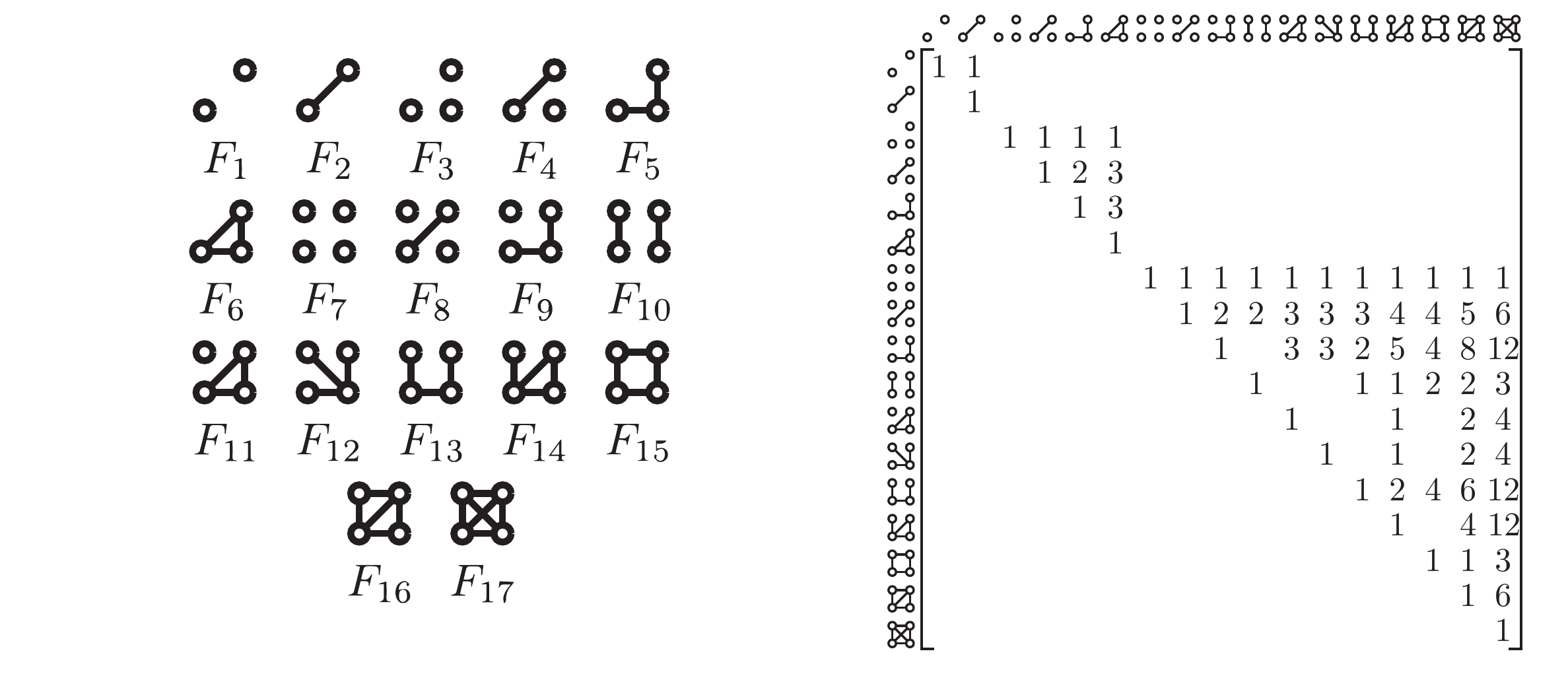}
    \caption{All graphs on at most four vertices, and $\mathcal{O}$, the overlap matrix. Zeros are omitted.}
    \label{fig:gabe1}
\end{figure}

\begin{table}[h!]
\begin{tabularx}{1\textwidth}{cY|cY|cY}
\toprule
\textbf{Graph} & \textbf{Formula}                          & \textbf{Graph} & \textbf{Formula}                             & \textbf{Graph} & \textbf{Formula}                             \\ \midrule
\parbox[c]{1em}{\centering\includegraphics[width=1em]{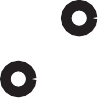}}    & $\binom{|V_G|}{2}$               & \parbox[c]{1em}{\centering\includegraphics[width=1em]{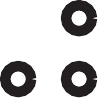}}    & $\binom{|V_G|}{3}$                  & \parbox[c]{1em}{\centering\includegraphics[width=1em]{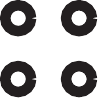}}    & $\binom{|V_G|}{4}$                  \\[.05in]
\parbox[c]{1em}{\centering\includegraphics[width=1em]{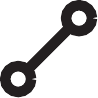}}    & $|E_G|$                          &\parbox[c]{1em}{\centering\includegraphics[width=1em]{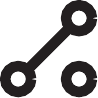}}    & $|E_G|(|V_G|-2)$                    & \parbox[c]{1em}{\centering\includegraphics[width=1em]{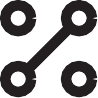}}    & $|E_G|\binom{|V_G|-2}{2}$           \\[.05in]
\parbox[c]{1em}{\centering\includegraphics[width=1em]{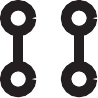}}   & $\binom{|E_G|}{2} - |H^{F_5}_G|$ & \parbox[c]{1em}{\centering\includegraphics[width=1em]{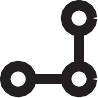}}    & $\sum_{v \in V_G} \binom{d^v_G}{2}$ & \parbox[c]{1em}{\centering\includegraphics[width=1em]{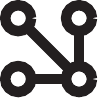}}   & $\sum_{v \in V_G} \binom{d^v_G}{3}$ \\[.07in]
\parbox[c]{1em}{\centering\includegraphics[width=1em]{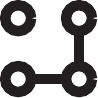}}    & $|H^{F_5}_G|(|V_G|-3)$           & \parbox[c]{1em}{\centering\includegraphics[width=1em]{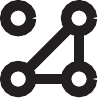}}   & $|H^{F_6}_G|(|V_G|-3)$              & -     & -                                   \\[.04in] \bottomrule
\end{tabularx}
\caption{Graphs and their corresponding subgraph count formulas.}
\label{tab:gabeform}
\end{table}

\subsubsection{Time and Space Complexity} 

Let $G'$ denote the graph represented by $\widetilde{E_G} \cup \{e_t\}$. We assume that $G'$ is stored as an adjacency list, where the list of neighbors for each vertex is stored in a sorted, tree-like structure. Thus, determining if two vertices are adjacent takes $O(\log b)$ time.

The diameter of each connected graph in $\mathcal{F}$ is 2. Thus, for an edge $e_t = (u_t, v_t)$, only vertices at most two hops away from $u_t$ or $v_t$ need to be visited. At most, three adjacency checks are needed to discover each connected graph. Hence, for a single edge, $2\left(\sum_{w \in N_{G'}(u_t)} d^w_{G'} + \sum_{w \in N_{G'}(v_t)} d^w_{G'}\right) \times 3\log{b} = O(b\log{b})$ time is taken to process one edge. Thus, checking the entire graph takes $O(b\log{b}|E_G|)$ time. 
$O(|V_G|)$ integers are stored to keep track of the degrees of each $v \in V_G$. 
Since each value can be accessed in $O(1)$ time, the counts for $F_5$ and $F_{12}$ can be updated each time an edge arrives in $O(1)$ time as well. 

It takes $O(1)$ time to compute the remaining estimates. Thus, the total runtime is $O(b\log{b}|E|)$. Storing the adjacency list and degree array takes $O(b+ |V_G|)$ space.

\subsubsection{Effect of Increasing $k$}

This method may be extended to further $k$ to create richer descriptors. This would require implementing algorithms to find connected components on $k$ vertices, deducing formulas to count the disconnected components, and constructing the overlap matrix to find the induced counts. However, obtaining counts for all subgraphs on $k$ vertices requires finding $k$-cliques, which have $\binom{k}{2}$ edges. The probability of detecting larger cliques in the stream will decrease with increasing $k$. 
Thus, increasing too much larger $k$ is likely to be unfeasible.

\subsection{MAEVE: Moments of Attributes Estimated on Vertices Efficiently}
\label{sec:maeve}

\textsc{NetSimile}~\cite{berlingerio2013network} proposed extracting local features for each vertex and aggregating them by taking various moments over their distribution. The features chosen by the authors are based on four social theories that allowed them to encompass the connectivity, transitivity, and reciprocity among the vertices and the control of information flow across graphs.

Similarly, we extract a subset of those features---chosen because they require at most one pass of the edge stream, listed in Table~\ref{tab:maeve}. As in \textsc{NetSimile}, the mean, standard deviation, skewness, and kurtosis for each feature are computed over the vertices. The only moment used in \textsc{NetSimile} ignored in our work is the median, left out to ensure that only one pass is needed over the vertices' features.

\subsubsection{Extracting Vertex Features}

For a graph $G$, and a vertex $v \in V_G$, $I_G(v)$ (see Table~\ref{tbl_common_notation} for description) is defined as the induced subgraph of $G$ formed by $\{v\} \cup N_G(v)$, i.e., $V_{I_G(v)} = N_G(v) \cup \{v\}$ and $E_{I_G(v)} = \{(u,v) | u,v \in V_{I_G(v)} \land (u,v) \in E_G\}$. Note that $I_G(v)$ is also referred to as the ``egonet'' of $v$. We define  $T_G(v)$ and $P_G(v)$ as the set of triangles that $v$ belongs to and the set of three-paths (paths on three vertices) where $v$ is an end-point, respectively. 
In Theorem~\ref{thm:maeve} (described below), we show that each feature described in Table~\ref{tab:maeve} can be calculated using values for  $d^v_G$, $|T_G(v)|$, and $|P_G(v)|$. Thus, the vertex counts for triangles are estimated for each vertex, as described in Sections~\ref{sec:est} and~\ref{sec:gabe}. Note that unlike in~\ref{sec:gabe}, the three-path estimates are not computed via the formula in Table~\ref{tab:gabeform} since this formula provides no information on the number of three-paths for each vertex. Moreover, the formula $\binom{d^v_G}{2}$ only provides us with the number of three-paths in which $v$ is the middle vertex. Thus, an edge-centric algorithm is employed for each vertex to estimate the number of three-paths it ends at via the stored sample.

\begin{table}[h!]
\centering
\begin{tabular}{cl}
\toprule
\textbf{Notation} & \textbf{Description} \\ \midrule
$I_G(v)$ & Subgraph induced on $\{v\} \cup N_G(v)$ \\ [.03in]
$T_G(v)$ & Number of triangles $v$ belongs to \\ [.03in]
$P_G(v)$ & Number of paths $v$ belongs to, as an endpoint \\  [.03in]
\bottomrule
\end{tabular}
\caption{Notation table for common terms used throughout Section~\ref{sec:maeve}.}
\label{tbl_common_notation}
\end{table}

\begin{table}[h!]
\begin{tabularx}{1\columnwidth}{cYYYc}
\toprule
Degree & Clustering Coefficient & Avg. Degree of $N_G(v)$ & Edges in $I_G(v)$ & Edges leaving $I_G(v)$ \\[.02in] 
\midrule
$d^v_G$ & $|T_G(v)|/\binom{d^v_G}{2}$ & $1+|P_G(v)|/d^v_G$ & $d^v_G + |T_G(v)|$ & $|P_G(v)| - 2|T_G(v)|$ \\[.12in]
\parbox[c]{0.15\columnwidth}{\centering\includegraphics[width=0.13\columnwidth]{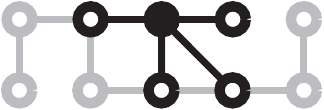}} & \parbox[c]{0.15\columnwidth}{\centering\includegraphics[width=0.13\columnwidth]{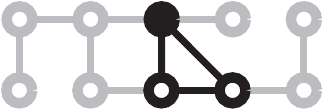}} & \parbox[c]{0.15\columnwidth}{\centering\includegraphics[width=0.13\columnwidth]{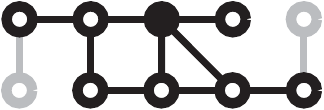}} & \parbox[c]{0.15\columnwidth}{\centering\includegraphics[width=0.13\columnwidth]{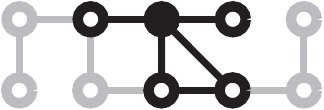}} & \parbox[c]{0.15\columnwidth}{\centering\includegraphics[width=0.13\columnwidth]{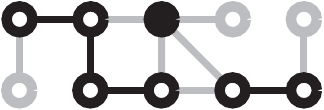}} \\[.04in]
\bottomrule
\end{tabularx}
\caption{Features extracted for each vertex, $v \in V_G$ for \textsc{maeve}, their formulae, and a figure highlighting the relevant edges. The filled-in vertex depicts $v$.}
\label{tab:maeve}
\end{table}

\begin{theorem}
\label{thm:maeve}
Each feature described in Table~\ref{tab:maeve} can be expressed in terms of $d^v_G$, $|T_G(v)|$, and $|P_G(v)|$.
\end{theorem}

\begin{proof}
Observe that the degree and clustering coefficient of a vertex is already written in terms of $d^v_G$ and $|T_G(v)|$. We will now show that the remaining three features can also be formulated in terms of $d^v_G$, $|T_G(v)|$, and $|P_G(v)|$.

\textit{Average Degree of Neighbors:} 
Consider a vertex $u \in N_G(v)$. For each $w \in N_G(v) \setminus \{v\} $, $w$ must end at a the three-path $(v,u,w)$. The only remaining edge for each $u \in N_G(v)$ is $v$ itself. Note that when summing over the degrees of all vertices in $N_G(v)$, $v$ appears once in each degree, and thereby $d^v_G$ times in total. Hence $|P_G(v)| + d^v_G =  \sum_{u \in N_G(v)} d^u_G$ and the average can be expressed as $\frac{1}{d^v_G}\sum_{u \in N_G(v)} d^u_G = 1 + \frac{|P_G(v)|}{d^v_G}$

\textit{Edges in $I_G(v)$:} 
Let $X \subseteq E_{I_G(v)}$ be the set of all edges in $E_{I_G(v)}$ that are incident on $v$, and $\overline{X} = E_{I_G(v)} \setminus X$ be the complement of $X$. Clearly, $E_{I_G(v)} = X \cup \overline{X}$ and $X \cap \overline{X} = \emptyset$. Thus, $|E_{I_G(v)}| = |X| + |\overline{X}|$. By definition, there are exactly $d^v_G$ edges incident on $v$, and each of them belongs to $E_{I_G(v)}$. Thus, $|X| = d^v_G$.

Now, consider any edge $(u,w) \in \overline{X}$. Recall that $u,w \not= v$, so, by the definition of $I_G(v)$, $u,w \in N_G(v)$. Thus, $(u,w)$ must be part of the triangle $\{(u,v), (v,w), (u,w)\}$. Since each edge in $\overline{X}$ forms a triangle incident on $v$, we have that $|\overline{X}| = |T_G(v)|$. Hence, $\left|E_{I_G(v)}\right| = |X| + |\overline{X}| = d^v_G + |T_G(v)|$.

\textit{Edges leaving $I_G(v)$:} 
Consider a three-path $(v,u)(u,w)$. Clearly, if $w \not\in N_G(v)$, $(u,w)$ must be an edge leaving $I_G(v)$. Thereby, the number of edges leaving $I_G(v)$ must be all three-paths starting at $v$ and ending at a vertex not in $N_G(v)$. Thus, we must account for all three-paths starting at $v$ that lie in $I_G(v)$. Clearly, if $w \in E_{I_G(v)}$, then the following three-paths are formed: $(v,u)(u,w)$, and $(v,w),(v,u)$. Hence, each triangle in $T_G(v)$ contributes twice to the number of paths in $P_G(v)$, and we can formulate the feature as $|P_G(v)| - 2|T_G(v)|$.
\end{proof}

Observe that each feature is a linear combination of our estimated variables, $|T_G(v)|$ and $|P_G(v)|$. Thus, we note that the features computed for each vertex are equal to the true value on expectation, as per Theorem~\ref{thm:unbiased} and the linearity of expectation.

\subsubsection{Time and Space Complexity} 
We assume the same adjacency list structure described in Section~\ref{sec:gabe}. Let $G'$ be the sampled graph. 
Three arrays of length $|V_G|$ are used to store the values of $d^v_G$, $P_G(v)$, and $|T_G(v)|$ for all $v \in V_G$. The degree of each vertex takes $O(1)$ time to update.
Let $e_t = (u_t, v_t)$ be the edge arriving at time $t$. Due to the sorted nature of our adjacency list, triangles incident on $e_t$ can be found by computing the intersection of $N_G(u_t)$ and $N_G(v_t)$ in $O(|N_G(u_t)| + |N_G(v_t)|)$ time. Counting three-paths also takes $O(|N_G(u_t)| + |N_G(v_t)|)$ time, as one pass over each neighborhood is required. Thus, the time taken to process each edge is $O(|N_G(u_t)| + |N_G(v_t)|) = O(b)$, and processing all edges takes $O(b|E_G|)$ time. After processing the entire edge stream, computing the moments over all arrays takes $O(|V_G|)$ time. Thus, the total runtime is $O(b|E_G| + |V_G|)$. The space complexity is $O(b + |V_G|)$, on account of storing $b$ edges in the adjacency list and a constant number of arrays of length $|V_G|$.

\subsection{SANTA: Spectral Attributes for Networks via Taylor Approximation}
\label{sec:santa}

For a graph $G = (V_G, E_G)$, let $A_G$ be its adjacency matrix. Let $D_G$ be a diagonal matrix where $D_G(i,i)$ is the degree of vertex $v_i \in V_G$. Let $\mathcal{L}_G = I_G - D^{-\frac{1}{2}}AD^{-\frac{1}{2}} $ be the normalized Laplacian of $G$ (see Table~\ref{tbl_notations_3} for notation description), where $I_G$ is the $|V_G| \times |V_G|$ identity matrix. Let $\lambda_k$ be the $k^{th}$ eigenvalue of $\mathcal{L}_G$, and $\Lambda_G = (\lambda_1, \lambda_2, \ldots, \lambda_{|V_G|})$ refers to the eigenspectrum of $\mathcal{L}_G$. In~\cite{tsitsulin2018netlsd}, Tsitsulin et al. present \textsc{NetLSD}: a descriptor based on the spectral properties of a graph. \textsc{NetLSD}'s descriptors are based on functions of the form  $\psi_j: \Lambda_G \rightarrow \mathbb{R} $ which map $\mathcal{L}_G$'s eigenspectrum to a real number, based on a parameter $j$. For a set of parameters $\{j_1, j_2, \ldots, j_m\}$, the vectors take the following form:
\begin{equation*}
\begin{bmatrix}
\psi_{j_1}(\Lambda_G) ~&
\psi_{j_2}(\Lambda_G) ~& 
\cdots ~& 
\psi_{j_{m-1}}(\Lambda_G) ~& 
\psi_{j_m}(\Lambda_G) 
\end{bmatrix}^\intercal
\end{equation*}

\begin{table}[h!]
\centering
\begin{tabular}{cl}
\toprule
\textbf{Notation} & \textbf{Description} \\ 
\midrule
$I_G$ & $|V_G| \times |V_G|$ identity matrix \\ [.04in]
$\Lambda_G$ & List of eigenvalues \\ [.04in]
$\mcl$ & Normalized Laplacian of $G$ \\ [.04in]
$\psi_j$ & Function that maps $\Lambda_G$ to a real number based on $j$ \\ [.04in]
$j$ & Parameter for $\psi_j$ \\  [.04in]
$n$& Exponent of $\mcl$
\\ \bottomrule
\end{tabular}
\caption{Notation table for common terms used throughout Section~\ref{sec:santa}.}
\label{tbl_notations_3}
\end{table}

The authors of \textsc{NetLSD} define six different functions based on two ``kernels'' and three normalization factors based on the eigenspectrums of complete graphs and their complements on $|V_G|$ vertices. Each function is of the form:
\begin{equation*}
    \psi_j(\Lambda_G) = \alpha \times \text{Re}\bigg( \sum_{\lambda_k \in \Lambda_G}   e^{-j\beta\lambda_k} \bigg)
\end{equation*} 

where $\alpha$ is a normalization factor dependent on $|V_G|$ and $j$, and $\beta \in \{1, i\}$. Each function has been mentioned in Table~\ref{tab:netlsd}. Note that $\beta = 1$ for Heat kernel, and $\beta = i = \sqrt{-1}$ for the Wave kernel. For small values of $j$, Tsitsulin et al. suggest approximating the functions using the Taylor expansion:
\begin{equation*}
    \alpha \sum_{k = 0}^{\infty} \frac{\trace((-j\beta \mcl)^k)}{k!} = \alpha \trace(I_G) - \alpha j\beta~ \trace(\mcl)  + \alpha\frac{( j\beta)^2}{2}\trace(\mcl^2) + \cdots
\end{equation*}

The authors of \textsc{NetLSD} discuss approximating $\psi_j(\Lambda_G)$ for small $j$ using three Taylor terms. By enumerating over subgraphs, we propose using the first five terms of the Taylor expansion to construct a descriptor similar to \textsc{NetLSD}'s for small values of $j$:

\begin{equation*}
\begin{multlined}
    \psi_j(\Lambda_G) = \alpha \text{Re} \bigg( \trace(I_G) -  j\beta~\trace(\mcl)  + \frac{( j\beta)^2}{2}\trace(\mcl^2)
    % \big
     - \frac{( j\beta)^3}{6}\trace(\mcl^3) + \frac{( j\beta)^4}{24}\trace(\mcl^4) \vphantom{\frac{( j\beta)^2}{2}\trace(\mcl^2)} \bigg)
\end{multlined}
\end{equation*}

In the remainder of this section, we discuss how subgraph enumeration can be used to compute $\trace(\mcl^n)$ for $n \leq 4$ and a two-pass algorithm that can approximate \textsc{NetLSD} using the estimation scheme discussed previously.

\subsubsection{Computing the Trace via Subgraph Enumeration} 

For an adjacency matrix $A_G$, $A_G^n(u,v)$ is the number of walks of length $n$ from $u$ to $v$. 
The $n^{th}$ product of the Laplacian behaves similarly with the added facts that: \textit{(1)} we must also consider the self-loops induced on each vertex due to the 1's in $\mcl$'s diagonal, and \textit{(2)} the value added to $\mcl^n(u,v)$ by a walk will be a product of the ``weights'' of each of its edges, as each entry in the Laplacian corresponds to the following:

\begin{equation*}
    \mcl(u, v) = 
\begin{cases}
1, & \text{if $u = v$ and $d_G^u > 0 $}\\
- \dfrac{1}{\sqrt{d_G^u d_G^v}}  & \text{if $(u,v) \in E_G$}\\
0  & \text{otherwise}
\end{cases}
\end{equation*}

Using these facts, one can assert the following:

\begin{theorem}
\label{thm:netlsd}
The value of $\trace\left(\mcl^n\right)$ can be computed for $n \in \{2,3,4\}$ by enumerating over all subgraphs on at most four vertices.
\end{theorem}

\begin{table}[h!]
\centering
\begin{tabular}{cccc}
\toprule
\textbf{Kernel} & \multicolumn{3}{c}{\textbf{Normalization}} \\ 
\midrule
 & None        & Empty       & Complete       \\ 
 [.1in]
 $\text{Heat}$ &  $\sum e^{-j\lambda}$ & $\frac{1}{|V_G|}\sum e^{-j\lambda}$ & $\dfrac{\sum e^{-j\lambda}}{1 + (|V_G|-1)e^{-j}}$         \\
 [.2in]
 $\text{Wave}$ &      $\text{Re}\left(\sum e^{-ij\lambda}\right)$       &       $\frac{1}{|V_G|}\text{Re}\left(\sum e^{-ij\lambda}\right)$      &          $\dfrac{\text{Re}\left(\sum e^{-ij\lambda}\right)}{1 + (|V_G|-1)\cos(j)}$ \\
 [.15in] 
 \bottomrule
\end{tabular}
\caption{Each cell represents a function of the form $\psi_j$ proposed by the authors of \textsc{NetLSD}. All summations are taken over each eigenvalue $\lambda \in \Lambda_G$.}
\label{tab:netlsd}
\end{table}

\begin{proof}
Clearly, $\mcl^n(u,u)$ is equal to the sum of the weights of all walks with $\leq n$ edges from a vertex $u$ to itself. Thus, it is sufficient to enumerate all such walks and sum the weight of each walk. We do this by enumerating all relevant subgraphs, then adding a term that accounts for the weight of each walk in the subgraph and the number of walks within it. The largest subgraph induced by a walk of length $n$ from a vertex to itself is a $n$-cycle, which has $n$ vertices.
The relevant subgraphs for each $n \in \{2,3,4\}$ are shown in Tables \ref{tab:terms1}, \ref{tab:terms2}, and \ref{tab:terms3}. Observe that each coefficient of each term is determined by the number of walks of length $n$ possible on the corresponding subgraph, and each term is determined by the product of the weights of the edges as specified by the definition of the Laplacian. 
\end{proof}

\begin{table}[h!]
    \centering
    \begin{tabular}{ccc|ccc}
    \toprule
         \textbf{Subgraph} & \textbf{Walks} & \textbf{Term} & \textbf{Subgraph} & \textbf{Walks} & \textbf{Term} \\ \midrule
         \parbox[c]{1cm}{\centering\includegraphics[width=1cm]{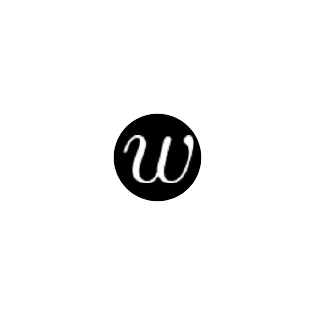}}               & $www$        & 1             & \parbox[c]{1cm}{\centering\includegraphics[width=1cm]{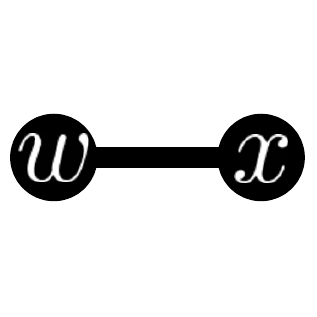}}               & $\begin{matrix}
    wxw & xwx
    \end{matrix} $        & $\dfrac{2}{d_G^w d_G^x}$  \\[0.2cm] \bottomrule
    \end{tabular}
    \caption{Subgraphs and terms relevant to computing $\trace\left(\mcl^2\right)$.}
    \label{tab:terms1}
\end{table}

\begin{table}[h!]
    \centering
    \resizebox{\textwidth}{!}{%
    \begin{tabular}{ccc|ccc}
    \toprule
         \textbf{Subgraph} & \textbf{Walks} & \textbf{Term} & \textbf{Subgraph} & \textbf{Walks} & \textbf{Term} \\ \midrule
         \parbox[c]{1cm}{\centering\includegraphics[width=1cm]{Figures/tame_g1.pdf}}               & $wwww$        & 1             &
    \parbox[c]{1cm}{\centering\includegraphics[width=1cm]{Figures/tame_g2.pdf}}               & $\begin{matrix}
    wwxw & wxww & wxxw \\
    xxwx & xwxx & xwwx
    \end{matrix} $        & $\dfrac{6}{d_G^w d_G^x}$  \\[0.5cm] \parbox[c]{1cm}{\centering\includegraphics[width=1cm]{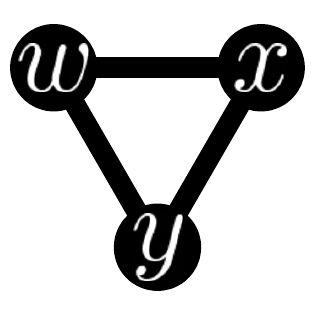}}           & $ \begin{matrix}
    wxyw & xywx & ywxy \\
    wyxw & xwyx & yxwy
    \end{matrix} $        & $-\dfrac{6}{d_G^w d_G^x d_G^y}$ & & &     \\[0.5cm] 
    \bottomrule
    \end{tabular}%
    }
    \caption{Subgraphs and terms relevant to computing $\trace\big(\mcl^3\big)$.}
    \label{tab:terms2}
\end{table}

\begin{table}[h!] 
\centering
\resizebox{\textwidth}{!}{%
\begin{tabular}{ccc|ccc}
\toprule
\textbf{Subgraph} & \textbf{Walks} & \textbf{Term} & \textbf{Subgraph} & \textbf{Walks} & \textbf{Term} \\ \midrule
\parbox[c]{1cm}{\centering\includegraphics[width=1cm]{Figures/tame_g1.pdf}}               & $wwwww$        & 1             &
\parbox[c]{1cm}{\centering\includegraphics[width=1cm]{Figures/tame_g2.pdf}}               & $\begin{matrix}
wwwxw & wwxxw & wxwww \\
wxxxw & wxxww & wwxww \\
xxxwx & xxwwx & xwxxx \\
xwwwx & xwwxx & xxwxx
\end{matrix} $       & $\frac{12}{d_G^w d_G^x}$           \\
[0.5cm] 
\parbox[c]{1cm}{\centering\includegraphics[width=1cm]{Figures/tame_g2.pdf}}               & $\begin{matrix}
wxwxw & xwxwx
\end{matrix} $        & $\frac{2}{(d_G^w d_G^x)^2}$           & 
\parbox[c]{1cm}{\centering\includegraphics[width=1cm]{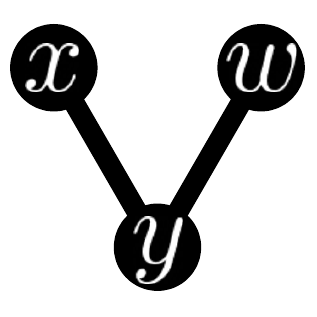}}                 & $\begin{matrix}
wyxyw & xywyx & yxywy & ywyxy
\end{matrix} $        & $\frac{4}{d_G^w d_G^y d_G^y d_G^x}$          \\
[0.5cm] 
\parbox[c]{1cm}{\vspace{0.5cm}\centering\includegraphics[width=1cm]{Figures/tame_g3.pdf}}           & $ \begin{matrix}
wxyww & wyyxw & xywxx & yxxwy \\ 
wwxyw & wyxxw & xxywx & yxwwy \\ 
wxxyw & xwyxx & xyywx & ywxyy \\ 
wxyyw & xxwyx & xywwx & yywxy \\ 
wyxww & xwwyx & yxwyy & ywwxy \\ 
wwyxw & xwyyx & yyxwy & ywxxy
\end{matrix} $        & $-\frac{24}{d_G^w d_G^x d_G^y}$ 
&
\parbox[c]{1cm}{\centering\includegraphics[width=1cm]{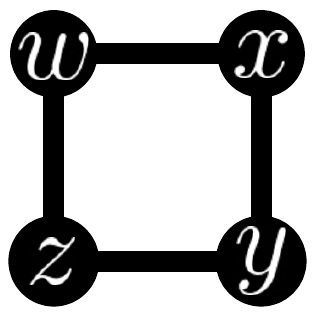}}          & $ \begin{matrix}
wxyzw & yzwxy & wzyxw & yxwzy \\
xyzwx & zwxyz & xwzyx & zyxwz 
\end{matrix} $        & $\frac{8}{d_G^w d_G^x d_G^y d_G^z}$    \\
[0.15in] 
\bottomrule
\end{tabular}%
}
\caption{All subgraphs and terms relevant to computing $\trace\left(\mcl^4\right)$.}
\label{tab:terms3}
\end{table}

\subsubsection{Computing the Descriptor on an Edge Stream}

We propose a two-pass algorithm to compute the descriptor.
In the algorithm's first pass, each vertex's degree is stored. In the second pass, the traces are computed; subgraphs are enumerated on the stream exactly as in the previous sections. When incrementing our count, the term to be added is multiplied by the probability of encountering it in the stream. We now show the validity of this method:

\begin{theorem}
\label{thm:netlsd2}
The approach proposed to approximate $\trace\left(\mcl^n\right)$ provides unbiased estimates.
\end{theorem}
\begin{proof}

We present a proof similar to the one presented in Theorem~\ref{thm:unbiased}.
Let $\tau_n$ be the estimate of $\trace\left(\mcl^n\right)$ provided by the algorithm described above.
Let $H^n_G$ be the set of subgraphs that are observed to increment $\tau_n$. For each $h \in H^n_G$, let $\delta_h$ be the term added to $\tau_n$ when $h$ is discovered in the stream. Recall from our prior discussion that $\trace\left(\mcl^n\right)$ can be defined as follows:
\begin{equation*}
    \trace\left(\mcl^n\right) = \sum_{h \in H^n_G} \delta_h 
\end{equation*}
Let $X_h$ be a random variable such that $X_h = \delta_h \times \frac{1}{p_t}$ if $h$ is discovered at the arrival of its last edge, and 0 otherwise, where $p_t$ is the probability of detecting $h$. Clearly, $\mathbb{E}\left[X_h \right] = (\delta_h/p_t) \times p_t = \delta_h$. We now analyze the expectation of $\tau_n$:
\begin{equation*}
    \mathbb{E}\left[\tau_n \right] = \mathbb{E}\bigg[\sum_{h \in H^n_G} X_h\bigg] = \sum_{h \in H^n_G} \mathbb{E}\left[X_h\right] = \sum_{h \in H^n_G} \delta_h = \trace\left(\mcl^n\right)
\end{equation*} 
\end{proof}

\subsubsection{Time and Space Complexity.} 
The computation performed is similar to the one in Section~\ref{sec:gabe}, with the extra step of storing the degrees in the first pass, which takes $O(|E_G|)$ time. Computing the descriptors takes $O(1)$ time. Thus, the time complexity is $O(b\log b |E_G|)$. Likewise, the space complexity is $O(b + |V_G|)$.

\section{Experimental Setup}\label{sec_experimental_evaluation}

This section outlines the experimental setup, including the dataset statistics, hyperparameter values, and data preprocessing. We also introduce state-of-the-art methods for comparing results with our proposed model. We show the visual representation of the proposed and the existing descriptor by converting them into 2-dimensional representations.
All experiments, except the ones on Malnet-TB, are performed on a single machine with 48 processors (2.50GHz Intel Xeon E5-2680v3) and 125 GB of memory.
The experiments on Malnet-TB are run on a single machine with 16 processors (3.70GHz Intel Xeon W-2145) and 32 GB of memory.
All algorithms were implemented\footnote{\url{https://git.io/JEQmI}} in C++ using an MPICHv3.2 backend. The code is built upon the Tri-Fly code, provided by Shin et al.~\cite{shin2018tri}. For each experiment, 25 processors simulate 1 master machine and 24 worker machines, and each embedding is computed once. 

\subsection{Hyperparameters}

Based on empirical observations, we use Canberra distance, $\left(d(\Vec{x},\Vec{y}) := \sum_{i = 1}^{d} \nicefrac{|\Vec{x}_i - \Vec{y}_i|}{|\Vec{x}_i|+|\Vec{y}_i|}\right)$, as the distance metric to measure approximation error for \textsc{gabe} and \textsc{maeve}. While $\ell_2$-distance metric is used to evaluate \textsc{santa}. We note that these error metrics are inline with those used in the literature, (c.f. ~\cite{berlingerio2013network,tsitsulin2018netlsd}).

As observed later, one achieves reasonable estimates for \textsc{santa} with $j \leq 1$. Thus, as in~\cite{tsitsulin2018netlsd}, we use 60 evenly-spaced values on the logarithmic scale within the range $[0.001,1]$ to construct the descriptors for \textsc{santa}. Note that when comparing \textsc{santa} to its actual values, the values produced by \textsc{NetLSD} are used. Thereby, the approximation error includes both the error introduced via subgraph estimation and the error via Taylor approximation.

\subsection{Datasets Statistics}

Our proposed model, along with the baselines and SOTA methods, are evaluated on various publicly available graph datasets, chosen primarily to showcase the efficacy of our model on large graphs.
Eight graph classification datasets were selected from the TUDataset~\cite{morris2020TUDataset} repository: DD~\cite{shervashidze2011weisfeiler}, CLB, RDT2, RDT5, and RDT12~\cite{yanardag2015deep}, OHSU~\cite{QiangDGLGZSGL21}, GHUB~\cite{karateclub}, FMM\footnote{\url{http://www.first-mm.eu/data.html}}~\cite{firstmmDBcite}. These datasets were selected due to the large size of the graphs within them relative to other datasets. The details for these datasets are provided in Table~\ref{tab:benchmark}.
Similarly, seven massive networks were selected from KONECT~\cite{konect} (i.e., Florida, USA, CiteSeer, Patent, Flickr, Stanford, and UK) to showcase the scalability of our models. The details of these graphs are provided in Table~\ref{tab:massive}.
\begin{enumerate}
    \item REDDIT graphs\footnote{\url{https://dynamics.cs.washington.edu/data.html}} were randomly sampled to construct a dataset to evaluate the approximation quality of our proposed methodology. Each graph represents a subreddit, wherein a vertex is a user within that subreddit, and an edge represents two users who have interacted within the subreddit.  RDT2, RDT5, and RDT12 are datasets of REDDIT graphs, as described earlier.
    \item DD is a bioinformatics dataset. Graphs in DD represent protein structures. A protein is represented as a graph, where the vertex represents amino acids, and there will be an edge between two vertices if they are connected less than 6 Angstroms apart. 
    \item OHSU is a bioinformatics dataset. OHSU graphs represent brain networks, wherein each vertex represents a region of the brain, and two regions are linked if they are correlated. 
    \item Graphs in CLB represent networks of researchers (each node is a single researcher) where edges represent that two researchers have collaborated. 
    \item GHUB graphs are social networks of developers (each developer is a node) who ``starred'' popular machine learning and web development repositories on Github to make the edges.
    \item Graphs in FMM represent 3D point clouds of household objects, wherein each vertex represents an object, and two objects share an edge when there are nearby.
    \item The Malnet dataset was used to test the efficacy of our models on large-scale classification tasks~\cite{freitas2021large}.
    The dataset aims to distinguish malware based on ``function call graphs.'' Thus, each graph represents a program, wherein each vertex is a function within the program, and each edge represents an ``inter-procedural'' call.
    In lieu of using the entire dataset, we used two subsets: Trojan and Benign. We consider the binary classification task of distinguishing between these two sets. 
    \item Florida (FO) and USA (US) datasets in KONECT are the road network graphs, where each edge represents a road, and each vertex represents the intersection of two or more roads.
    \item CiteSeer (CS) and Patent (PT) datasets in KONECT are the citation networks, where vertices represent documents and connected vertices represent documents that reference each other.
    \item Flicker (FL) dataset in KONECT is the friendship network, where vertices represent users on social networks, and two vertices are connected if the users are ``friends".
    \item Stanford (SF) and UK 2002 (U2) datasets in KONECT are the hyperlink networks, where vertices represent webpages, and connected vertices represent webpages that link to each other.
\end{enumerate}

\begin{table}[h!]
\centering
\begin{tabular}{lccccc}
\toprule
\textbf{Dataset}      & \textbf{Graphs} & \textbf{Classes} & $\mathbf{\max|V_G|}$ & $\mathbf{\max|E_G|}$ & \textbf{Avg. Deg.} \\ 
\midrule
FMM       & $41$     & $11$      & $5037$    & $21774$ & 4.50  \\ [.04in]
OHSU             & $79$     & $2$       & $171$     & $1646$ & 4.33  \\  [.04in]
DD             & $1178$   & $2$       & $5748$     & $14267$ & 4.98 \\ [.04in]
RDT2    & $2000$   & $2$       & $3782$    & $4071$ & 2.34 \\ [.04in]
RDT5  & $4999$   & $5$       & $3648$    & $4783$ & 2.25 \\ [.04in]
CLB           & $5000$   & $3$       & $492$     & $40120$ & 37.39 \\ [.04in]
RDT12 & $11929$  & $11$      & $3782$    & $5171$  & 2.28 \\  [.04in]
GHUB           & $12725$  & $2$       & $957$     & $9336$  & 3.20 \\ [.04in]
 Malnet-TB           & $258373$  &  $2$     &   $551873$   &  $1639647$  &  4.15  \\   [.04in]
\bottomrule
\end{tabular}
\caption{Descriptions of each dataset used in our work for graph classification. For each dataset, we list the number of graphs, the number of classes, the largest order and size of a graph within the dataset, and the average degree across all graphs in the dataset.}
\label{tab:benchmark}
\end{table}

\begin{table}[h!]
\centering
\begin{tabular}{lcccp{5.9cm}}
\toprule
\textbf{Graph} & $\mathbf{|V_G|}$     & $\mathbf{|E_G|}$     & \textbf{Type}  & \textbf{Description} \\ 
\midrule
Florida (FO)        & 1070376              & 1343951              & \multirow{2}{*}{Road}      & \multirow{2}{=}{\small Vertices are the intersections of two or more roads, edges represent roads} \\
USA (US)       & 23947347             & 28854312             &                            & \\[0.08in]
CiteSeer (CS)       & 384054               & 1736145              & \multirow{2}{*}{Citation}  & \multirow{2}{=}{\small Vertices are documents and are connected if one document references the other} \\
Patent (PT)         & 3774768              & 16518937             &                            & \\[0.08in]
\multirow{2}{*}{Flickr (FL)}         & \multirow{2}{*}{2302925}              & \multirow{2}{*}{22838276}             & \multirow{2}{*}{Friendship}                 & \multirow{2}{=}{\small Vertices represent users on social networks and are connected if the users are ``friends''}       \\[0.2in]
Stanford (SF)       & 281903               & 1992636              & \multirow{2}{*}{Hyperlink} & \multirow{2}{=}{\small Vertices represent webpages and vertices are connected if one webpage links to the other}     \\
UK 2002 (U2) & 18483186             & 261787258            &                            &                                                                                                               \\ \bottomrule
\end{tabular}
\caption{Massive networks from KONECT listed alongside the number of vertices and edges, and descriptions.}
\label{tab:massive}
\end{table}

For the preprocessing step, we convert each graph into an edge list. Duplicate edges and possible self-loops are removed from the list. If required, each vertex is relabelled to lie in the range $[0, |V_G|-1]$. Finally, the list is randomly shuffled to ensure that the input stream is unbiased. 

\subsection{Existing State-of-the-Art Descriptors}

We compare our models to the following state-of-the-art (SOTA) methods:
  
\par\smallskip\noindent{\textbf{\textsc{NetLSD}~\cite{tsitsulin2018netlsd}:}} \textsc{NetLSD} represents a graph based on the eigenspectrum of the graph's Laplacian. Euclidean distance is used to compare embeddings, as suggested by the authors in their work. We report the best accuracy for each of the six variants of \textsc{NetLSD}.
\par\smallskip\noindent{\textbf{\textsc{feather}~\cite{rozemberczki2020feather}:}} \textsc{feather}'s descriptors are aggregated over characteristic function descriptors of each node in a graph. The default hyperparameters are used to construct the descriptors. We report the best accuracy for each of the three variants proposed.
\par\smallskip\noindent{\textbf{\textsc{sf}~\cite{lara2018simple}:}} The authors of \textsc{sf} proposed a ``simple'' baseline algorithm based on the eigenspectrum of a graph's Laplacian. As suggested by the authors, the ``embedding dimension'' is set to the average number of nodes of each graph within a dataset.

\begin{remark}
Since no distance was suggested for \textsc{feather} or \textsc{sf}, we compute results on Euclidean and Canberra distances and report the best accuracy.
\end{remark}

\begin{remark}
Note that our models have no direct competitors, as no other graph classification paradigm is constructed to run under our proposed constraints. Despite this, we compare our model with the SOTA methods to show its effectiveness in terms of scalability.
\end{remark}

\subsection{Data Visualization}

To visually compare different descriptors, we plot them using $t$-distributed Stochastic Neighbour Embedding ($t$-SNE)~\cite{van2014accelerating}. Given the graph descriptors, $t$-SNE computes a $2$-dimensional representation of the feature vectors. 
Figure~\ref{fig_tsne_plots_dd} shows the $t$-SNE based visualization on DD dataset for our methods (\textsc{santa}, \textsc{gabe}, and \textsc{maeve} on $25\%$ and $50\%$ budget) and \textsc{NetLSD}. Observe that as we increase the budget, the class-wise separation of the data becomes more prominent. Moreover, \textsc{santa} shows the most similar representation of data to \textsc{NetLSD}.

\begin{figure}[h!]
  \centering
  \begin{subfigure}{.33\textwidth}
  \centering
  \includegraphics[scale = 0.17] {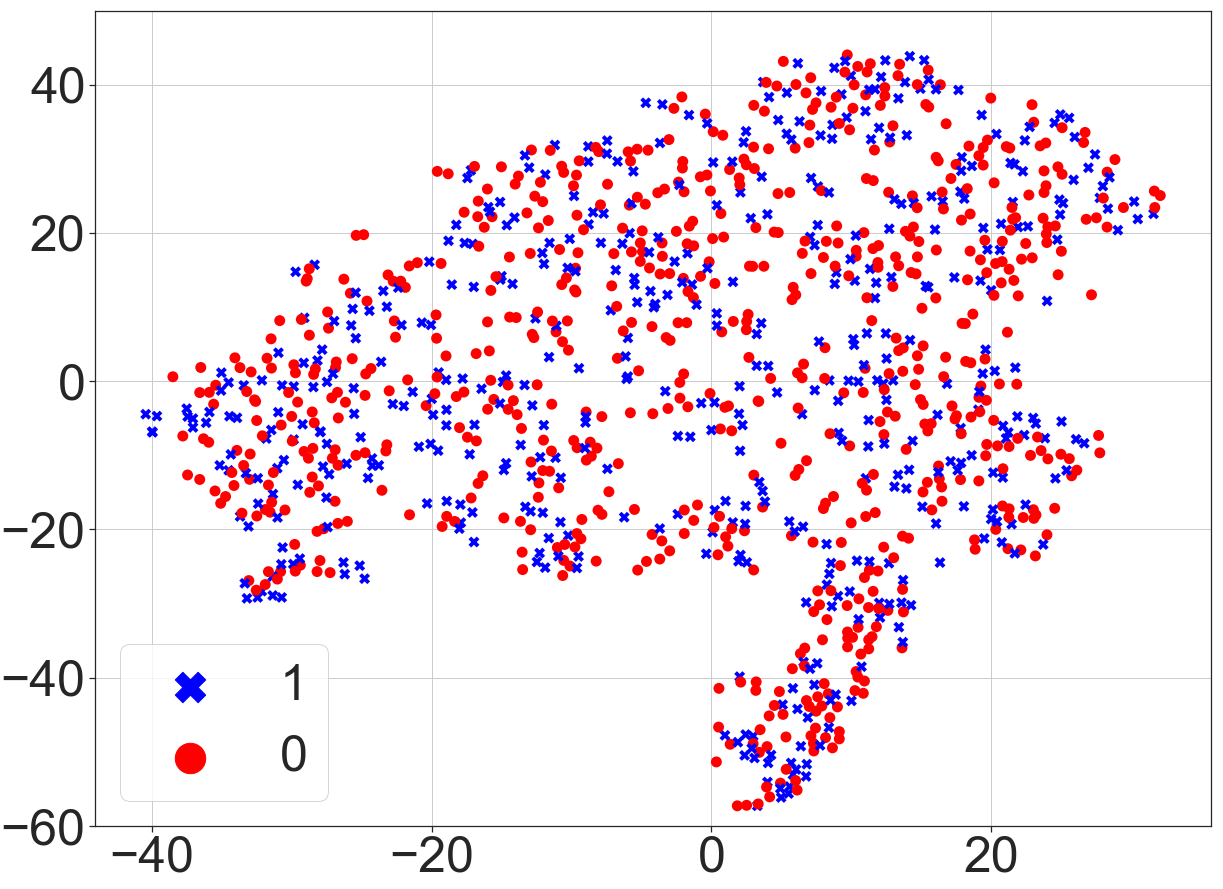}
  \caption{\textsc{maeve} 25\%}
  
  \end{subfigure}%
   \begin{subfigure}{0.33\textwidth}
  \centering
  \includegraphics[scale = 0.17] {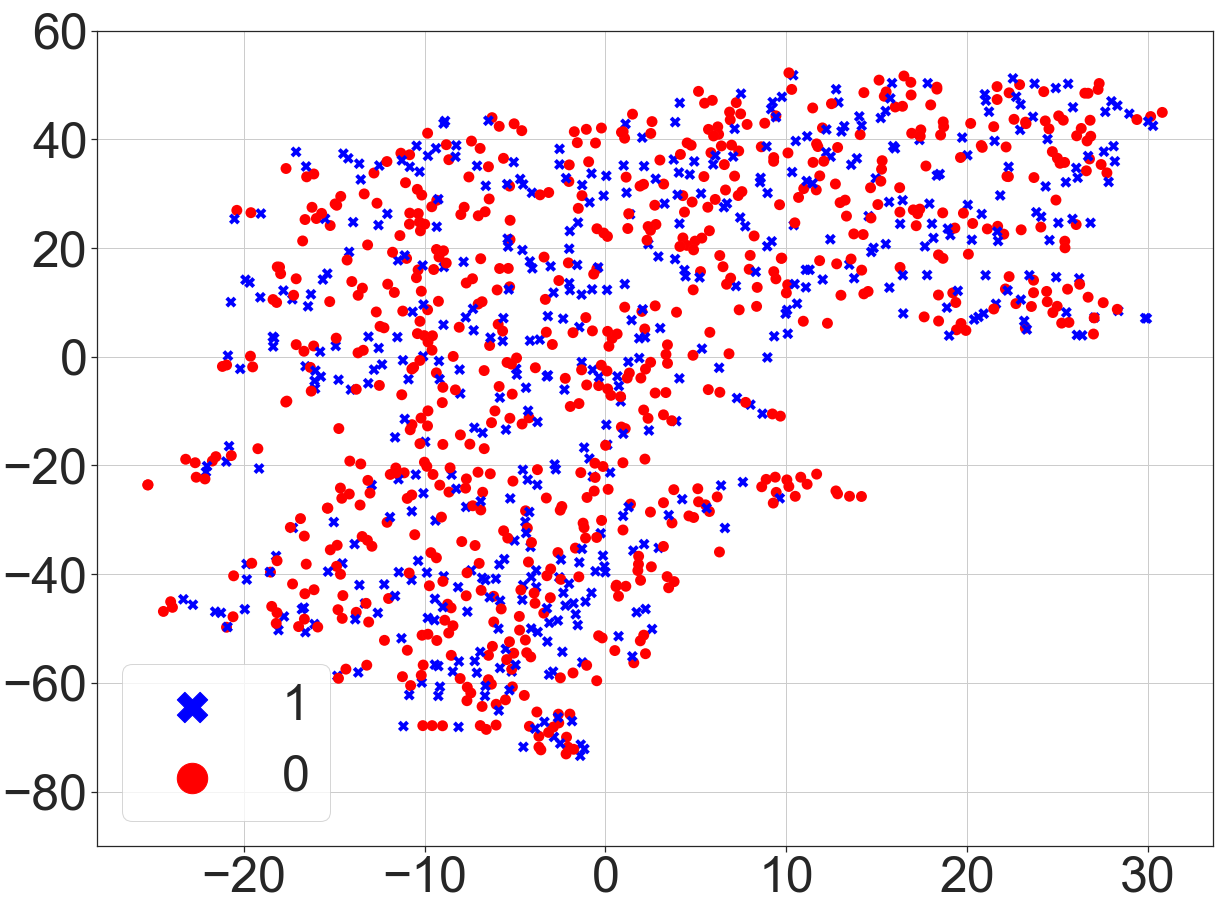}
  \caption{\textsc{maeve} 50\%}
  
  \end{subfigure}%
  \begin{subfigure}{0.33\textwidth}
  \centering
  \includegraphics[scale = 0.17] {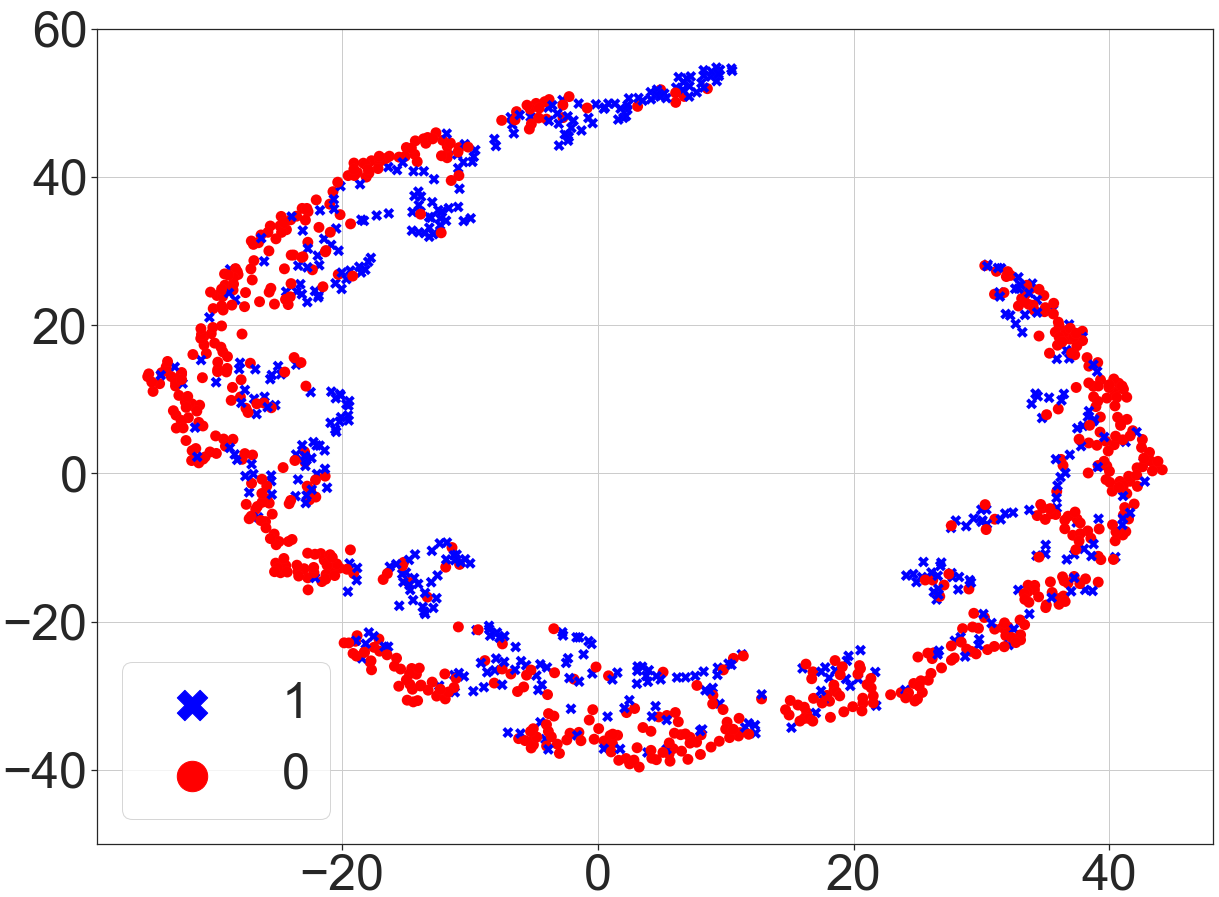}
  \caption{\textsc{NetLSD-hc}}
  \end{subfigure}%
  \\
  \begin{subfigure}{.33\textwidth}
  \centering
  \includegraphics[scale = 0.17] {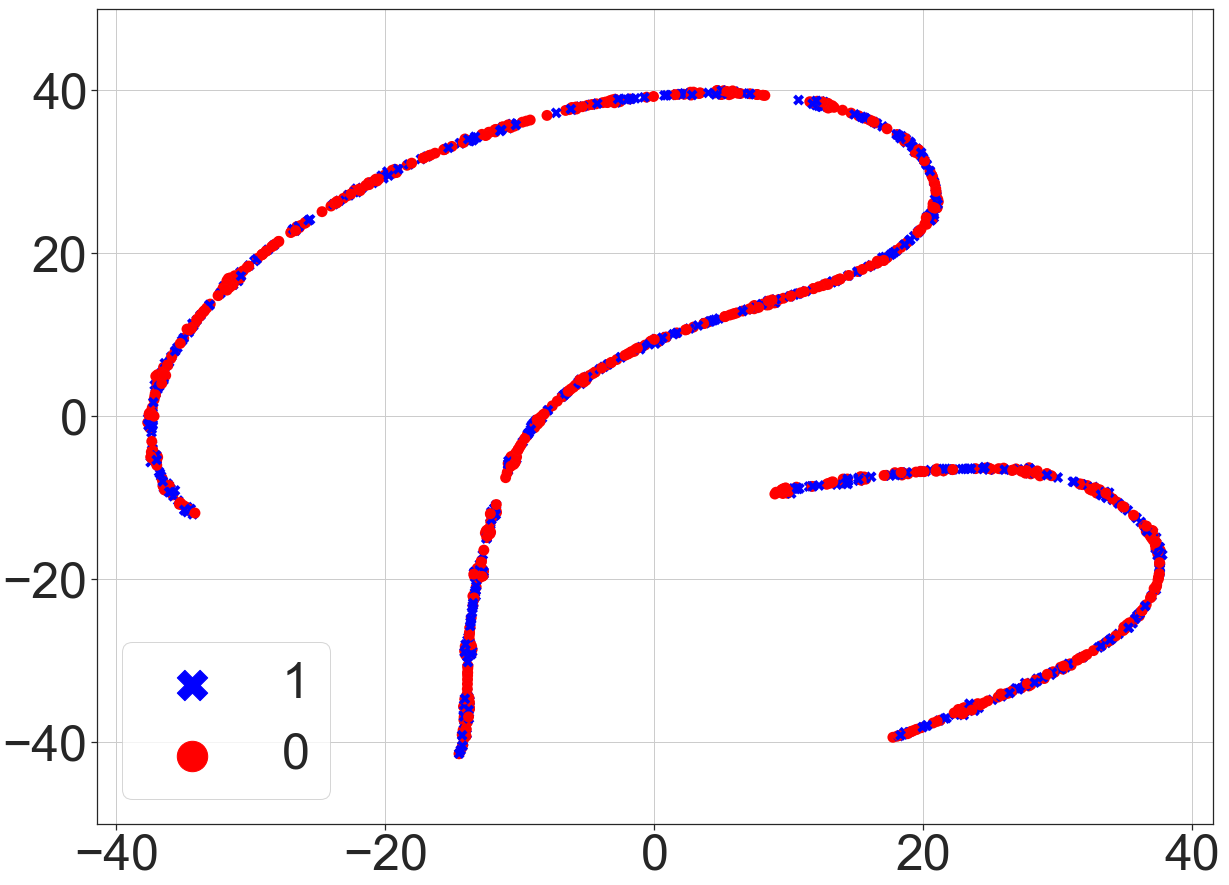}
  \caption{\textsc{gabe} 25\%}
  
  \end{subfigure}%
   \begin{subfigure}{0.33\textwidth}
  \centering
  \includegraphics[scale = 0.17] {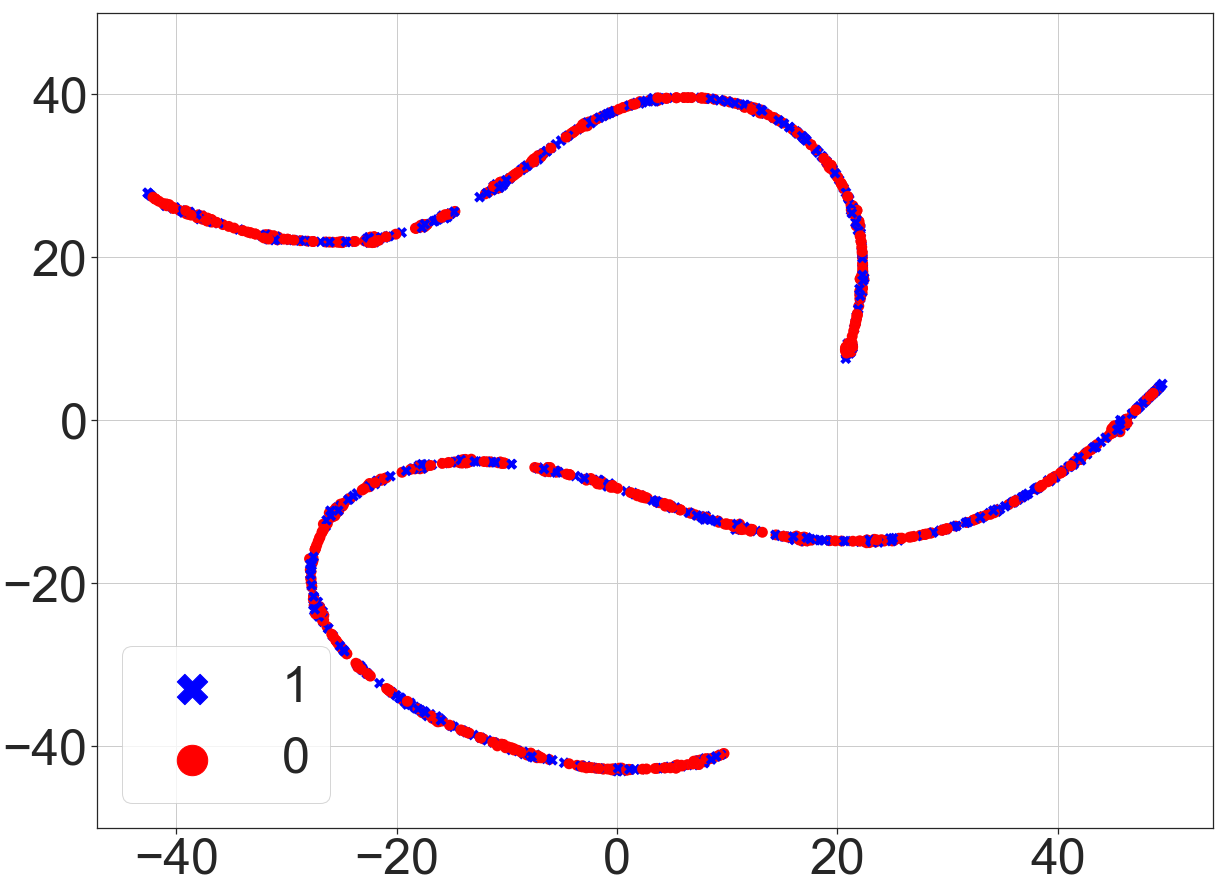}
  \caption{\textsc{gabe} 50\%}
  
  \end{subfigure}%
  \begin{subfigure}{0.33\textwidth}
  \centering
  \includegraphics[scale = 0.17] {tsne_csv/NETLSD-heatComplete0_DD_tsne_plot.png}
  \caption{\textsc{NetLSD-hc}}
  \end{subfigure}%
  \\
  \begin{subfigure}{.33\textwidth}
  \centering
  \includegraphics[scale = 0.17] {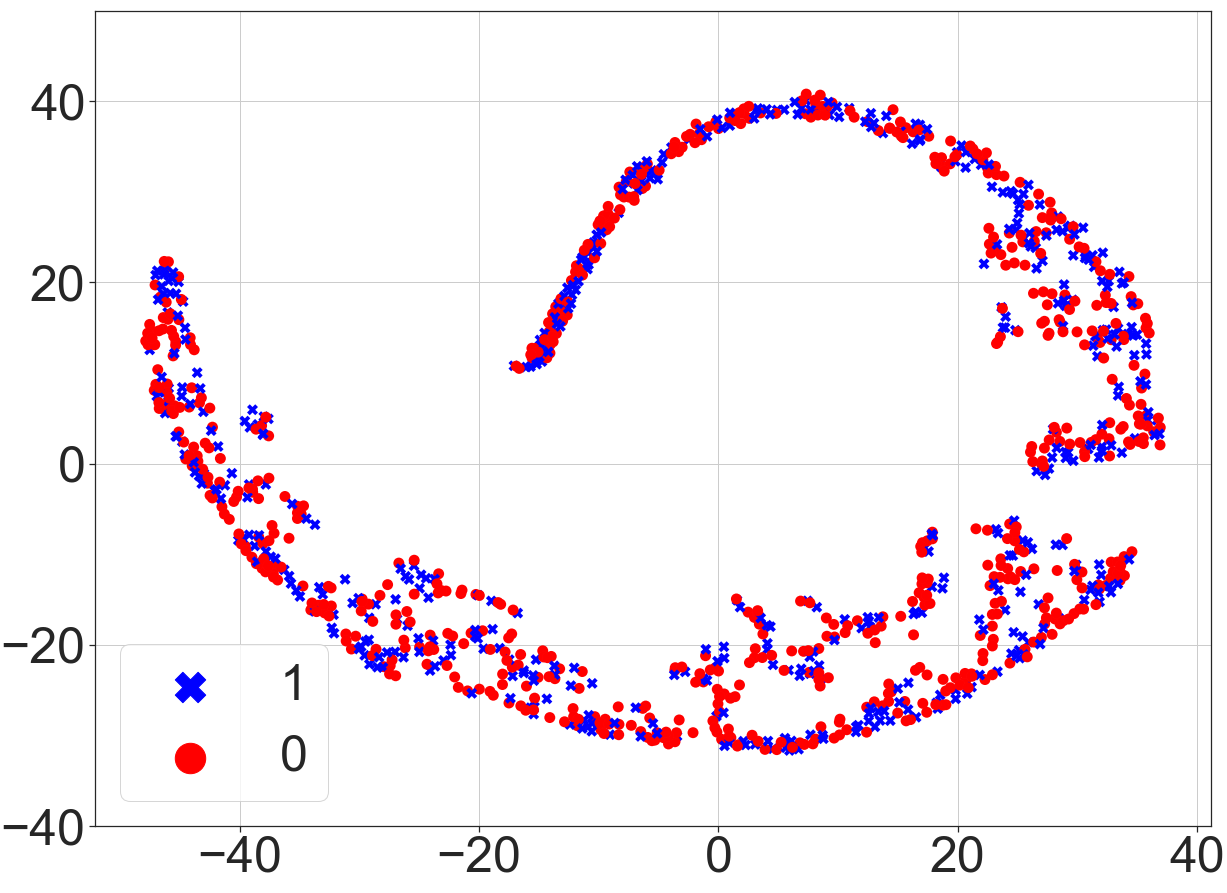}
  \caption{\textsc{santa-hc} 25\%}
  
  \end{subfigure}%
   \begin{subfigure}{0.33\textwidth}
  \centering
  \includegraphics[scale = 0.17] {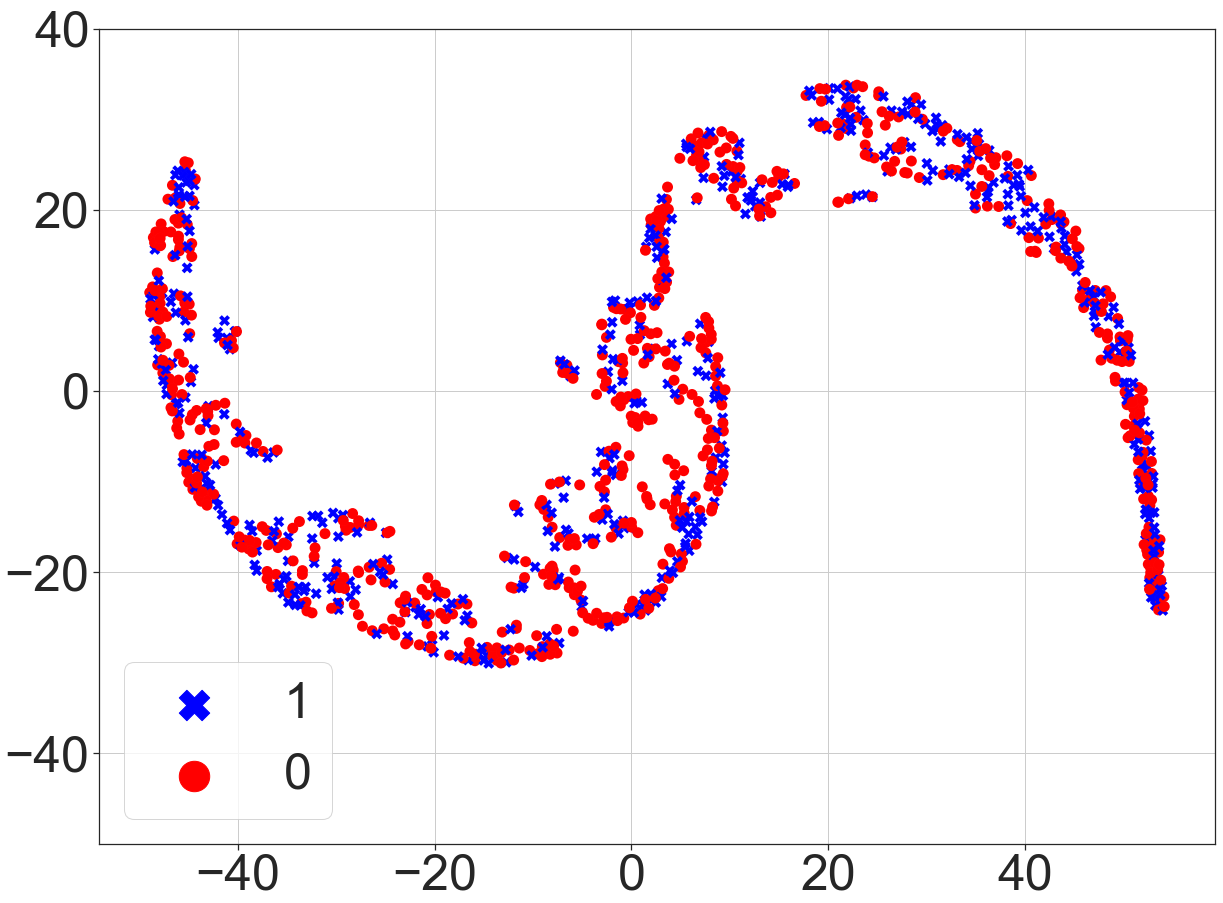}
  \caption{\textsc{santa-hc} 50\%}
  
  \end{subfigure}%
  \begin{subfigure}{0.33\textwidth}
  \centering
  \includegraphics[scale = 0.17] {tsne_csv/NETLSD-heatComplete0_DD_tsne_plot.png}
  \caption{\textsc{NetLSD-hc}}
  \end{subfigure}%
\caption{t-SNE plots for different descriptors and budgets on DD dataset. Legends show true classes. The figure is best seen in color.} 

 \label{fig_tsne_plots_dd}
\end{figure}

\section{Results and Discussion}\label{sec:experiments}

In this section, we report the results of our experiments and show the changes in approximation performance with varying values of the budget $b$. We also report the accuracies of classifiers learned on these descriptors. Furthermore, we demonstrate the scalability of the corresponding streaming algorithms. 
The distance between the exact and the approximate descriptor (output of the algorithms) is referred to as the approximation error of the algorithms.  
Note that in the figures ahead, \textsc{santa-xy} corresponds to the \textsc{santa} descriptor variant with kernel \textsc{x} (\textsc{h} or \textsc{w}) and normalization \textsc{y} (\textsc{n}, \textsc{e},  or \textsc{c}).

\subsection{Approximation Quality}

In this section, we test the approximation quality of our descriptors.
We uniformly sampled 1,000 graphs of size 10,000 to 50,000 from REDDIT, representing interactions in various ``sub-reddits''.

\subsubsection{Effect of number of Taylor terms for \textsc{SANTA}}

We first show in Figure~\ref{fig:plots4} how increasing the number of Taylor terms affects the approximation quality of \textsc{santa} with respect to $j$. For 1000 linearly spaced values of $j \in [0.001,1]$, the relative error
(defined as $\|x-\hat{x}|/x$, where $x$ is the real value and $\hat{x}$ is the approximation) 
across 1000 REDDIT graphs is averaged and plotted. Observe that increasing the number of Taylor terms allows us to better approximate values for larger $j$, enabling us to use a greater range of $j$. 

Note that there is no need to check this for each normalization since the normalization is canceled out when computing relative error.
Also, note that the values produced by four terms are ignored for the wave kernel since the values introduced in the fourth term are imaginary and are not used in the descriptor.

\begin{figure}[h!]
    \centering
    \subcaptionbox{Heat kernel}{%
    \resizebox {0.3\columnwidth} {!} {%
    \centering
    \begin{tikzpicture}
    \begin{axis}[
        xlabel={$j$},
        ylabel={Average Error},
        xtick={0,0.2,0.4,.6,.8,1},
        ytick={0,0.1,0.2,0.3,0.4,0.5},
        legend pos=north west,
        ymajorgrids=true,
        grid style=dashed,
        label style={font=\Large},
        mark size=0.000001pt,
    ]
    \addplot table [mark=none,x=x, y=t2, col sep=comma] {heat.tex}; \addlegendentry{Three Terms}
    \addplot table [mark=none,x=x, y=t3, col sep=comma] {heat.tex}; \addlegendentry{Four Terms}
    \addplot table [mark=none,x=x, y=t4, col sep=comma] {heat.tex}; \addlegendentry{Five Terms}
    \end{axis}
    \end{tikzpicture}
    }}%
    \subcaptionbox{Wave kernel}{%
    \resizebox {0.3\columnwidth} {!} {%
    \centering
    \begin{tikzpicture}
    \begin{axis}[
        xlabel={$j$},
        ylabel={Average Error},
        xtick={0,0.2,0.4,.6,.8,1},
        ytick={0,0.1,0.2,0.3,0.4,0.5},
        legend pos=north west,
        ymajorgrids=true,
        grid style=dashed,
        label style={font=\Large},
        mark size=0.001pt,
    ]
    \addplot table [x=x, y=t2, col sep=comma] {wave.tex}; 
    \addlegendentry{Three Terms}
    \addplot table [mark=circle,x=x, y=t4, col sep=comma] {wave.tex}; 
    \addlegendentry{Five Terms}
    \end{axis}
    \end{tikzpicture}
    }}%
    \caption{Average relative error for $j \in [0.001,1]$ of \textsc{santa} with varying number of Taylor terms.}
    \label{fig:plots4}
\end{figure}
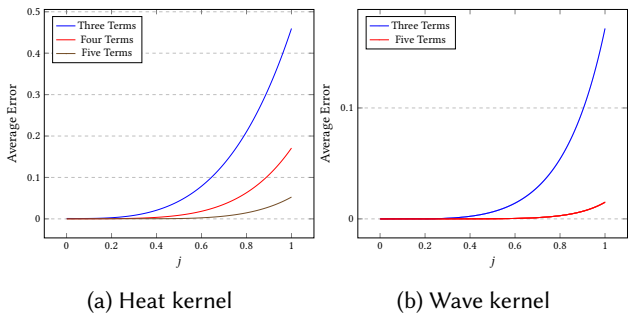

\subsubsection{Effect of increasing the budget for each descriptor}

Figure~\ref{fig:plots3} shows that the average approximation error across the sampled graphs decreases as the budget increases. Observe that normalized versions of \textsc{santa} can achieve very low errors even with small values of $b$. Unfortunately, un-normalized variants of \textsc{santa} have very large errors and would likely be unfruitful in practical settings.

\begin{figure}[h!]
    \centering
    \subcaptionbox{\textsc{gabe}}{%
    \resizebox {0.245\columnwidth} {!} {%
    \centering
    \begin{tikzpicture}
    \begin{axis}[
        xlabel={Budget [\% of $|E_G|$]},
        ylabel={Avg. Distance [\num{1e-1}]},
        xtick={0,10,20,30,40,50},
        ytick={0,8,16},
        ymin=0, ymax=16,
        legend pos=north east,
        ymajorgrids=true,
        grid style=dashed,
        label style={font=\Large}
    ]
    \addplot[
        color=blue,
        mark=square,
        ]
        coordinates {
        (5,15.15)(10,7.09)(15,4.07)(20,2.79)(25,2.01)(30,1.5)(35,1.15)(40,1.01)(45,0.77)(50,0.63)
        };
    \end{axis}
    \end{tikzpicture}
    }}%
    \subcaptionbox{\textsc{maeve}}{%
    \resizebox {0.245\columnwidth} {!} {%
    \centering
    \begin{tikzpicture}
    \begin{axis}[
        xlabel={Budget [\% of $|E_G|$]},
        ylabel={Avg. Distance [\num{1e-1}]},
        ymin=0, ymax=32, 
        xtick={0,10,20,30,40,50},
        ytick={0,16,32},
        legend pos=north east,
        ymajorgrids=true,
        grid style=dashed,
        label style={font=\Large},
            ]
    \addplot[
        color=blue,
        mark=square,
        ]
        coordinates {
        (5,31.73)(10,21.98)(15,16.6)(20,12.61)(25,9.92)(30,7.67)(35,5.94)(40,4.64)(45,3.58)(50,2.72)
        };
    \end{axis}
    \end{tikzpicture}
    }}%
    \subcaptionbox{\textsc{santa-hn}}{%
    \resizebox {0.245\columnwidth} {!} {%
    \centering
    \begin{tikzpicture}
    \begin{axis}[
        xlabel={Budget [\% of $|E_G|$]},
        ylabel={Avg. Distance [\num{1e1}]},
        ymin=0, ymax=3, 
        xtick={0,10,20,30,40,50},
        ytick={0,1,2,3},
        legend pos=north east,
        ymajorgrids=true,
        grid style=dashed,
        label style={font=\Large},
    ]
    \addplot[
        color=blue,
        mark=square,
        ]
        coordinates {
        (5,2)(10,2.1)(15,1.99)(20,1.94)(25,1.91)(30,1.88)(35,1.87)(40,1.86)(45,1.85)(50,1.84)
        };
    \end{axis}
    \end{tikzpicture}
    }}%
    \subcaptionbox{\textsc{santa-he}}{%
    \resizebox {0.245\columnwidth} {!} {%
    \centering
    \begin{tikzpicture}
    \begin{axis}[
        xlabel={Budget [\% of $|E_G|$]},
        ylabel={Avg. Distance [\num{1e-3}]},
        ymin=0, ymax=4, 
        xtick={0,10,20,30,40,50},
        ytick={0,2,4},
        legend pos=north east,
        ymajorgrids=true,
        grid style=dashed,
        label style={font=\Large},
    ]
    \addplot[
        color=blue,
        mark=square,
        ]
        coordinates {
        (5,2.91)(10,3.05)(15,2.89)(20,2.81)(25,2.76)(30,2.73)(35,2.71)(40,2.69)(45,2.68)(50,2.67)
        };
    \end{axis}
    \end{tikzpicture}
    }}%
    \\
    % \hfill
    \subcaptionbox{\textsc{santa-hc}}{%
    \resizebox {0.245\columnwidth} {!} {%
    \centering
    \begin{tikzpicture}
    \begin{axis}[
        xlabel={Budget [\% of $|E_G|$]},
        ylabel={Avg. Distance [\num{1e-3}]},
        ymin=0, ymax=10,
        xtick={0,10,20,30,40,50},
        ytick={0,5,10},
        legend pos=north east,
        ymajorgrids=true,
        grid style=dashed,
        label style={font=\Large},
    ]
    \addplot[
        color=blue,
        mark=square,
        ]
        coordinates {
        (5,7.56)(10,7.94)(15,7.52)(20,7.31)(25,7.19)(30,7.1)(35,7.05)(40,7.01)(45,6.97)(50,6.95)
        };
    \end{axis}
    \end{tikzpicture}
    }}%
    \subcaptionbox{\textsc{santa-wn}}{%
    \resizebox {0.245\columnwidth} {!} {%
    \centering
    \begin{tikzpicture}
    \begin{axis}[
        xlabel={Budget [\% of $|E_G|$]},
        ylabel={Avg. Distance},
        ymin=0, ymax=10, 
        xtick={0,10,20,30,40,50},
        ytick={0,5,10},
        legend pos=north east,
        ymajorgrids=true,
        grid style=dashed,
        label style={font=\Large},
    ]
    \addplot[
        color=blue,
        mark=square,
        ]
        coordinates {
        (5,7.98)(10,8.98)(15,7.91)(20,7.38)(25,7.05)(30,6.84)(35,6.7)(40,6.59)(45,6.5)(50,6.44)
        };
    \end{axis}
    \end{tikzpicture}
    }}%
    \subcaptionbox{\textsc{santa-we}}{%
    \resizebox {0.245\columnwidth} {!} {%
    \centering
    \begin{tikzpicture}
    \begin{axis}[
        xlabel={Budget [\% of $|E_G|$]},
        ylabel={Avg. Distance [\num{1e-3}]},
        ymin=0, ymax=2,
        xtick={0,10,20,30,40,50},
        ytick={0,1,2},
        legend pos=north east,
        ymajorgrids=true,
        grid style=dashed,
        label style={font=\Large},
    ]
    \addplot[
        color=blue,
        mark=square,
        ]
        coordinates {
        (5,1.14)(10,1.29)(15,1.12)(20,1.04)(25,1)(30,0.96)(35,0.94)(40,0.93)(45,0.91)(50,0.9)
        };
    \end{axis}
    \end{tikzpicture}
    }}%
    \subcaptionbox{\textsc{santa-wc}}{%
    \resizebox {0.245\columnwidth} {!} {%
    \centering
    \begin{tikzpicture}
    \begin{axis}[
        xlabel={Budget [\% of $|E_G|$]},
        ylabel={Avg. Distance [\num{1e-3}]},
        ymin=0, ymax=3, 
        xtick={0,10,20,30,40,50},
        ytick={0,1,2,3},
        legend pos=north east,
        ymajorgrids=true,
        grid style=dashed,
        label style={font=\Large},
    ]
    \addplot[
        color=blue,
        mark=square,
        ]
        coordinates {
        (5,2.02)(10,2.27)(15,1.99)(20,1.85)(25,1.77)(30,1.71)(35,1.67)(40,1.64)(45,1.62)(50,1.61)
        };
    \end{axis}
    \end{tikzpicture}
    }}%
    \caption{Approximation errors with increasing budget $(b)$ for \textsc{gabe}, \textsc{maeve}, and all variants of \textsc{santa}.}
    \label{fig:plots3}
\end{figure}
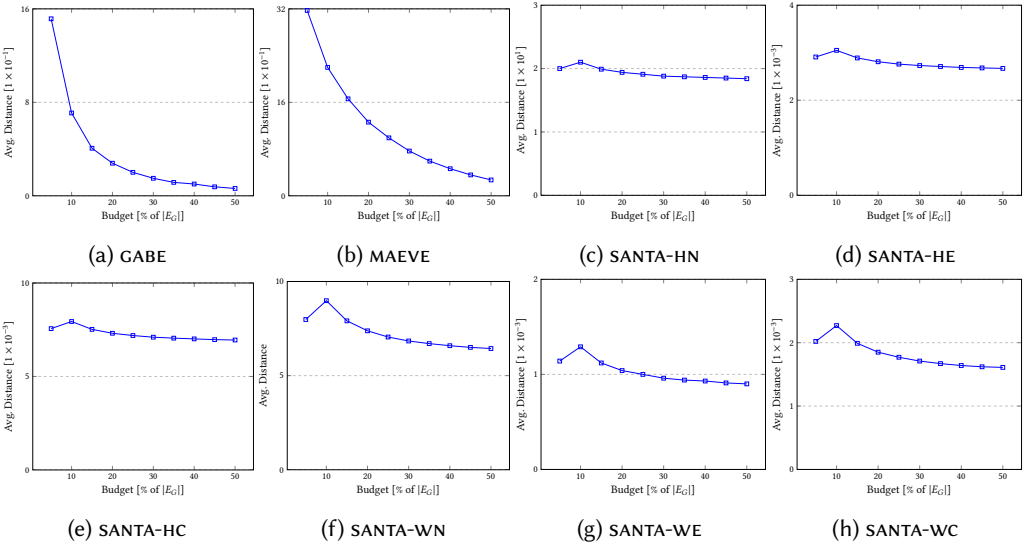

\subsection{Graph Classification}

We opted for a Nearest Neighbor classifier as in Tsitulin et al.~\cite{tsitsulin2018netlsd} work on \textsc{NetLSD}. 10-fold cross-validation was performed for ten random splits of the dataset. The average accuracy for each fold is reported. Note that only two folds are used for FMM because each class has a small number of samples. The descriptors are computed for our models by using  25\% and 50\% of the number of edges of each graph.

\subsubsection{Results on Different Variants of SANTA}

In Table~\ref{tab:ressanta}, we compare all variants of \textsc{santa} to find out which one works best. It is clear that \textsc{santa-hc} often provides the best results. For this reason, and because it has the lowest error across all variants in Figure~\ref{fig:plots3}, we recommend \textsc{santa-hc} for practical usage and compare it to other descriptors in the coming section.

In this same table, we show the results provided by \textsc{NetLSD} when using the same values for $j$. Despite the error added by the Taylor approximation and budgeted sampling, \textsc{santa} provides results comparable to \textsc{NetLSD}.

We observe better results for the datasets OHSU and FMM when the budget is smaller, sometimes more significant than those provided by \textsc{NetLSD}. 
We believe that due to the small size of these datasets, the noise added when approximating the embeddings is not eliminated by the classifier. Thus, we do not recommend using \textsc{santa} on smaller datasets without a larger budget; otherwise, the classifier may not be able to generalize.

Table~\ref{tab:res} compares the classification accuracy of our proposed models and the benchmark descriptors.
Despite using only a fraction of edges, our proposed descriptors provide results competitive to descriptors that have access to the entire graph in seven of the eight classification datasets. Unfortunately, \textsc{santa} is unable to compete with its competitors in most cases, despite giving results near to \textsc{NetLSD} when used on the same values of $j$ (see Table~\ref{tab:ressanta}).

\subsubsection{Comparing SANTA to SLaQ}

We report the comparison of the accuracy achieved by \textsc{santa} (across all variants) with that of SLaQ, a method introduced by Titsulin et al.~\cite{tsitsulin2020just} to approximate NetLSD, in Table~\ref{tab:santavsslaq}. For DD and CLB datasets, we observe that \textsc{santa} outperforms SLaQ. For RDT5 and RDT12 datasets, although SLaQ outperforms SANTA, the predictive accuracy of SLaQ is not significantly higher compared to \textsc{santa}, despite SLaQ keeping the entire graph in memory.

\begin{table}[h!]
\begin{tabular}{@{}cccccc@{}}
\toprule
\textbf{Method}                                  & \textbf{Budget}        & \textbf{DD} & \textbf{CLB} & \textbf{RDT5} & \textbf{RDT12} \\ 
\midrule
SLaQ & $|E_G|$ & 66.77 & 58.76 & \textbf{35.48} & \textbf{25.31} \\ 
\midrule
\multirow{2}{*}{\textsc{santa}} & $\nicefrac{1}{4}|E_G|$ & \textbf{68.16}       & 63.80        & \textbf{35.32}         & \textbf{24.68}          \\ 
& $\nicefrac{1}{2}|E_G|$ & 66.83       & \textbf{64.90}        & \textbf{34.62}         & 23.89          \\ 
\bottomrule
\end{tabular}
\caption{Comparing reported classification accuracy of \textsc{santa} and SLaQ. Results within 1\% of the best have been bold-faced.}
\label{tab:santavsslaq}
\end{table}

\subsubsection{Performance on Large Classification Tasks}

To showcase the practical usage of our proposed methods, we performed graph classification on the Malnet-TB dataset on a computer with relatively weaker hardware. In this case, we used ten workers and $\nicefrac{1}{10}|E_G|$ as our budget. The time taken and classification accuracy is provided in Table~\ref{fig:malnet}.
Note that all of our proposed models can process $ 260$K graphs with up to $550$K vertices and $1.6$M edges in $\approx 1 \nicefrac{1}{2}$ days.

\begin{table}[h!]
\centering
\resizebox{0.98\columnwidth}{!}{%
\begin{tabular}{@{}ccccccccccc@{}}
\toprule
\textbf{Variant} &  \textbf{Method} &  \textbf{Budget} &  \textbf{DD} &  \textbf{CLB} &  \textbf{RDT2} &  \textbf{RDT5} &  \textbf{RDT12} &  \textbf{OHSU} &  \textbf{GHUB} &  \textbf{FMM} \\ 
\midrule
\multirow{3}{*}{\textsc{hn}} &  \multirow{2}{*}{\textsc{santa}} &
  $\nicefrac{1}{4}|E_G|$ &  66.22 &  61.90 &  76.02 &  \textbf{35.12} &  22.38 &  54.50 &  54.88 &  26.80 \\ [.04in]
 &   &  $\nicefrac{1}{2}|E_G|$ &  66.03 &  62.59 &  75.88 &  \textbf{34.39} &  22.21 &  54.50 &  54.78 &  26.80 \\ [.04in]
%   \cmidrule(lr){3-11}
 &  \textsc{NetLSD}\textsuperscript{$\star$} &
  $|E_G|$ &  66.44 &  63.29 &  75.82 &  33.50 &  21.74 &  56.98 &  55.75 &  27.14 \\ [.04in]
  \cmidrule(lr){2-11}
\multirow{3}{*}{\textsc{he}} &  \multirow{2}{*}{\textsc{santa}} &
  $\nicefrac{1}{4}|E_G|$ &  63.98 &  63.80 &  63.77 &  \textbf{34.90} &  21.42 &  \textbf{69.96} &  54.56 &  \textbf{39.70} \\ [.04in]
 &   &  $\nicefrac{1}{2}|E_G|$ &  65.76 &  \textbf{64.90} &  64.33 &  34.22 &  21.91 &  66.82 &  \textbf{55.11} &  20.00 \\ [.04in]
 &  \textsc{NetLSD}\textsuperscript{$\star$} &
  $|E_G|$ &  60.75 &  64.08 &  61.98 &  29.44 &  19.47 &  52.66 &  57.19 &  21.49 \\ [.04in]
   \cmidrule(lr){2-11}
\multirow{3}{*}{\textsc{hc}} &  \multirow{2}{*}{\textsc{santa}} &  $\nicefrac{1}{4}|E_G|$ &  \textbf{68.16} &  63.44 &  \textbf{79.14} &  \textbf{35.32} &  \textbf{24.68} &  67.98 &  \textbf{55.99} &  \textbf{38.76} \\ [.04in]
 &   &  $\nicefrac{1}{2}|E_G|$ &  66.83 &  63.50 &  \textbf{78.34} &  \textbf{34.62} &  \textbf{23.89} &  58.25 &  \textbf{55.61} &  23.74 \\ [.04in]
 &  \textsc{NetLSD}\textsuperscript{$\star$} &  $|E_G|$ &  65.99 &  64.77 &  75.96 &  37.02 &  25.12 &  55.95 &  55.03 &  35.39 \\ [.04in]
  \cmidrule(lr){2-11}
\multirow{3}{*}{\textsc{wn}} &  \multirow{2}{*}{\textsc{santa}} &  $\nicefrac{1}{4}|E_G|$ &  66.70 &  62.49 &  75.68 &  \textbf{35.08} &  22.76 &  55.30 &  \textbf{55.32} &  26.80 \\ [.04in]
 &   &  $\nicefrac{1}{2}|E_G|$ &  66.63 &  63.15 &  75.57 &  \textbf{34.53} &  22.81 &  55.30 &  \textbf{55.30} &  26.80 \\ [.04in]
 &  \textsc{NetLSD}\textsuperscript{$\star$} &  $|E_G|$ &  66.19 &  63.01 &  75.64 &  33.40 &  22.23 &  54.14 &  58.08 &  28.60 \\ [.04in]
  \cmidrule(lr){2-11}
\multirow{3}{*}{\textsc{we}} &  \multirow{2}{*}{\textsc{santa}} &
  $\nicefrac{1}{4}|E_G|$ &  61.55 &  62.52 &  65.10 &  34.09 &  21.66 &  67.59 &  \textbf{55.04} &  24.37 \\ [.04in]
 &   &  $\nicefrac{1}{2}|E_G|$ &  61.02 &  62.04 &  64.90 &  33.56 &  21.27 &  64.32 &  54.06 &  11.27 \\ [.04in]
 &  \textsc{NetLSD}\textsuperscript{$\star$} &
  $|E_G|$ &  59.35 &  64.46 &  62.14 &  26.99 &  19.05 &  60.61 &  58.20 &  15.08 \\ [.04in]
   \cmidrule(lr){2-11}
\multirow{3}{*}{\textsc{wc}} &  \multirow{2}{*}{\textsc{santa}} &
  $\nicefrac{1}{4}|E_G|$ &  64.15 &  61.25 &  74.26 &  31.45 &  21.43 &  58.48 &  \textbf{55.12} &  24.38 \\ [.04in]
 &   &  $\nicefrac{1}{2}|E_G|$ &  61.81 &  62.47 &  74.64 &  31.79 &  21.46 &  58.12 &  54.67 &  11.74 \\ [.04in]
 &  \textsc{NetLSD}\textsuperscript{$\star$} &  $|E_G|$ &  64.81 &  62.97 &  75.10 &  29.39 &  21.56 &  47.93 &  56.60 & 19.01 \\ [.04in] 
  \bottomrule
\end{tabular}%
}
\caption{Classification accuracy (in \%) using nearest neighbor classifier across all datasets for all variants of \textsc{santa}, as well as \textsc{NetLSD} modified to use the same values for $j$. Results within 1\% of the best across all \textsc{santa} variants have been bold-faced.}
\label{tab:ressanta}

\resizebox{0.98\columnwidth}{!}{%
\begin{tabular}{p{1.8cm}cccccccccc}
\toprule
 \textbf{Approach} & \textbf{Method} &  \textbf{Budget} &  \textbf{DD} &  \textbf{CLB} &  \textbf{RDT2} &  \textbf{RDT5} &  \textbf{RDT12} &  \textbf{OHSU} &  \textbf{GHUB} &  \textbf{FMM} \\ 
\midrule
 \multirow{3}{*}{Benchmark} & \textsc{NetLSD} &  $|E_G|$ &  \textbf{70.36} &   \textbf{74.27} &   82.85 &   41.23 &   30.90 &   \textbf{73.79} &   55.73 &   27.14 \\ [.04in]
 & \textsc{feather} &  $|E_G|$ &  63.57 &   73.14 &   83.22 &   \textbf{43.09} &   \textbf{34.33} &   62.77 &   60.95 &   26.81 \\ [.04in]
 & \textsc{sf} &  $|E_G|$ &  62.84 &   72.82 &   82.38 &   \textbf{42.36} &   30.80 &   59.50 &   57.01 &   29.00 \\  [.04in]
\midrule
 \multirow{8}{1.8cm}{Proposed Descriptors} & 
 \multirow{2}{*}{\textsc{maeve}} &  $\nicefrac{1}{4}|E_G|$ &  59.44 &   68.42 &   85.04 &   41.15 &   32.57 &   49.07 &   \textbf{61.99} &   12.90 \\ [.04in]
 &  &  $\nicefrac{1}{2}|E_G|$ &  61.26 &   70.95 &   \textbf{86.15} &   41.53 &   \textbf{33.69} &   47.12 &   \textbf{61.81} &   14.63 \\ [.04in]
  \cmidrule(lr){2-11}
 & \multirow{2}{*}{\textsc{gabe}} &  $\nicefrac{1}{4}|E_G|$ &  65.23 &   63.62 &   84.65 &   41.10 &   32.18 &   44.30 &   \textbf{61.88} &   27.37 \\ [.04in]
 &  &  $\nicefrac{1}{2}|E_G|$ &  69.08 &   65.23 &   \textbf{85.35} &   40.63 &   32.96 &   41.02 &   \textbf{62.72} &   25.35 \\ [.04in]
 \cmidrule(lr){2-11}
 & \multirow{2}{*}{\textsc{santa-hc}} &  $\nicefrac{1}{4}|E_G|$ &  68.16 &   63.44 &   79.14 &   35.32 &   24.68 &   67.98 &   55.99 &   \textbf{38.76} \\ [.04in]
 &  &  $\nicefrac{1}{2}|E_G|$ &  66.83 &   63.50 &   78.34 &   34.62 &   23.89 &   58.25 &   55.61 &   23.74 \\ [.04in]

\bottomrule
\end{tabular}%
}
\caption{Accuracy of the nearest neighbor classifier on different datasets, descriptors, and benchmark methods. Results within 1\% of the best have been bold-faced.}
\label{tab:res}
\end{table}

\begin{table}[h!]
\resizebox{0.98\columnwidth}{!}{%
    \begin{tabular}{cccP{1.2cm}P{1.1cm}P{1.1cm}P{1.1cm}P{1.1cm}P{1.1cm}}
    \toprule
    
    & \multirow{2}{*}{\textbf{GABE}} &  \multirow{2}{*}{\textbf{MAEVE}} &  \textbf{SANTA HN} &  \textbf{SANTA HE} &  \textbf{SANTA HC} &  \textbf{SANTA WN} &  \textbf{SANTA WE} &  \textbf{SANTA WC} \\
    \toprule
   Accuracy & 79.52 & 79.93  & 74.60 & 69.38 & 69.47 & 74.89 & 69.39 & 69.40  \\ \midrule
   Avg. Time [s] &  0.54 & 0.45 &  0.36 & 0.36 & 0.36 & 0.36 & 0.36 & 0.36  \\
   Max Time [min] &  66.73 & 37.93 & 33.82 & 33.82 & 33.82 & 33.82 & 33.82 & 33.82  \\
   Total Time [hr] & 38.59  & 32.41 & 25.63 & 25.63 & 25.63 & 25.63 & 25.63 & 25.63  \\
   \bottomrule
    \end{tabular}
    }
\caption{Results on Malnet-TB for \textsc{gabe}, \textsc{maeve}, and all variants of \textsc{santa} with $b = \nicefrac{1}{10}|E_G|$. We report the accuracy, the average and maximum amount of time taken for each graph in the dataset, and the total amount of time taken.}
    \label{fig:malnet}
\end{table}

\begin{table}[h!]
    \centering
    \resizebox{0.98\columnwidth}{!}{%
    \begin{tabular}{ccccP{1.2cm}P{1.1cm}P{1.1cm}P{1.1cm}P{1.1cm}P{1.1cm}}
    \toprule
    
    & & \multirow{2}{*}{\textbf{GABE}} &  \multirow{2}{*}{\textbf{MAEVE}} &  \textbf{SANTA HN} &  \textbf{SANTA HE} &  \textbf{SANTA HC} &  \textbf{SANTA WN} &  \textbf{SANTA WE} &  \textbf{SANTA WC} \\
    \toprule
   \multirow{2}{*}{PT} & Time [min] & 0.52  &  0.77 & 1.11 & 1.11 & 1.11 & 1.11 & 1.11 & 1.11 \\
   & Distance & 3.36 & 5.11 & 2.38 & 6.31 & 1.35 & 4.36 & 1.72 & 1.14 \\
   \cmidrule(lr){3-10}
   \multirow{2}{*}{FL} & Time [min] & 5.48  & 3.48 & 4.96 & 4.96 & 4.96 & 4.96 & 4.96 & 4.96 \\
   & Distance & 2.77 & 5.09 & 1.14 & 4.95 & 1.02 & 2.43 & 1.59 & 1.04 \\
   \cmidrule(lr){3-10}
   \multirow{2}{*}{US} & Time [min] & 0.63  & 1.21 & 1.99 & 1.99 & 1.99 & 1.99 & 1.99 & 1.99 \\
   & Distance & 5.24 & 11.39 & 18.30 & 7.64 & 1.92 & 6.32 & 0.37 & 0.28 \\
   \cmidrule(lr){3-10}
   \multirow{2}{*}{U2} & Time [min] & 10.58  & 9.05 & 20.61 & 20.61 & 20.61 & 20.61 & 20.61 & 20.61 \\
   & Distance & 6.48 & 9.61 & - & - & - & - & - & - \\
   \cmidrule(lr){3-10}
   \multirow{2}{*}{FO} & Time [min] & 0.05 & 0.16 & 0.08 & 0.08 & 0.08 & 0.08 & 0.08 & 0.08 \\
   & Distance & 2.07 & 6.67 & 0.75 & 6.96 & 1.76 & 0.21 & 0.27 & 0.21 \\
   \cmidrule(lr){3-10}
   \multirow{2}{*}{CS} & Time [min] & 0.21 & 0.19 & 0.20 & 0.20 & 0.20 & 0.20 & 0.20 & 0.20 \\
   & Distance & 1.07 & 3.09 & 0.19 & 4.97 & 1.03 & 0.39 & 1.52 & 1.01 \\
   \cmidrule(lr){3-10}
   \multirow{2}{*}{SF} & Time [min] & 8.35 & 4.39 & 3.92 & 3.92 & 3.92 & 3.92 & 3.92 & 3.92 \\
   & Distance & 1.10 & 3.55 & 0.20 & 6.98 & 1.50 & 0.35 & 1.86 & 1.23 \\
        \bottomrule
    \end{tabular}
    }
\caption{Approximation error and time taken for \textsc{gabe}, \textsc{maeve}, and all variants of \textsc{santa} with $b = 100000$. Accuracy results for U2 have been omitted since the graph was too large to obtain true values.}

    \label{fig:plots1}
\resizebox{0.98\columnwidth}{!}{%
    \begin{tabular}{ccccP{1.2cm}P{1.1cm}P{1.1cm}P{1.1cm}P{1.1cm}P{1.1cm}}
    \toprule
    
    & & \multirow{2}{*}{\textbf{GABE}} &  \multirow{2}{*}{\textbf{MAEVE}} &  \textbf{SANTA HN} &  \textbf{SANTA HE} &  \textbf{SANTA HC} &  \textbf{SANTA WN} &  \textbf{SANTA WE} &  \textbf{SANTA WC} \\
    \toprule
   \multirow{2}{*}{PT} & Time [min] &  0.84 & 1.13 & 1.38 & 1.38 & 1.38 & 1.38 & 1.38 & 1.38 \\
   & Distance & 2.65 & 3.14 & 2.38 & 6.32 & 1.35 & 4.36 & 1.72 & 1.14 \\
   \cmidrule(lr){3-10}
   \multirow{2}{*}{FL} & Time [min] & 101.37 & 12.62 & 45.27 & 45.27 & 45.27 & 45.27 & 45.27 & 45.27 \\
   & Distance & 2.66 & 3.95 & 1.17 & 5.08 & 1.05 & 2.40 & 1.56 & 1.03 \\
   \cmidrule(lr){3-10}
   \multirow{2}{*}{US} & Time [min] & 0.90 & 1.34 & 1.94 & 1.94 & 1.94 & 1.94 & 1.94 & 1.94 \\
   & Distance & 4.84 & 10.08 & 18.35 & 7.66 & 1.92 & 6.30 & 0.36 & 0.28 \\
   \cmidrule(lr){3-10}
   \multirow{2}{*}{U2} & Time [min] & 17.23 & 19.18 & 26.88 & 26.88 & 26.88 & 26.88 & 26.88 & 26.88 \\
   & Distance & 3.56 & 7.73 & - & - & - & - & - & - \\
   \cmidrule(lr){3-10}
   \multirow{2}{*}{FO} & Time [min] & 0.12 & 0.17 & 0.14 & 0.14 & 0.14 & 0.14 & 0.14 & 0.14 \\
   & Distance & 1.16 & 2.52 & 0.74 & 6.96 & 1.76 & 0.21 & 0.27 & 0.21 \\
   \cmidrule(lr){3-10}
   \multirow{2}{*}{CS} & Time [min] & 1.88 & 0.71 & 0.83 & 0.83 & 0.83 & 0.83 & 0.83 & 0.83 \\
   & Distance & 1.01 & 0.86 & 0.19 & 4.97 & 1.03 & 0.39 & 1.52 & 1.00 \\
   \cmidrule(lr){3-10}
   \multirow{2}{*}{SF} & Time [min] & 174.54 & 29.10 & 54.37 & 54.37 & 54.37 & 54.37 & 54.37 & 54.37 \\
   & Distance & 1.04 & 1.58 & 0.20 & 6.99 & 1.50 & 0.35 & 1.85 & 1.23 \\
        \bottomrule
    \end{tabular}
    }
\caption{Approximation error and time taken for \textsc{gabe}, \textsc{maeve}, and all variants of \textsc{santa} with $b = 500000$. Accuracy results for U2 have been omitted since the graph was too large to obtain true values.}
    \label{fig:plots2}
\end{table}

\subsection{Scaling to Large Real-world Networks} 

In this section, we show the scalability of our proposed descriptors by running them on large real-world networks. For this purpose, we ran our algorithms on the networks listed in Table~\ref{tab:massive}. For each graph, descriptors were estimated for $b \in \{100000,500000\}$. In Table~\ref{fig:plots1} and ~\ref{fig:plots2}, we show the wall-clock time taken and the distance between the real and approximate vectors. Note that the lower values are better.

Note that to compute the real embeddings for \textsc{santa}, one would have to compute the eigenspectrum of each graph. Due to the intractability of this method, we approximate the true embeddings by approximating the eigenvalues using the largest and smallest eigenvalues of the Laplacian of each graph, as proposed in~\cite{tsitsulin2018netlsd}. As per the authors' suggestion, we attempted to obtain 150 eigenvalues from each end of the spectrum. While this was not possible for all graphs, a minimum of 50 eigenvalues were used for each end, i.e., at least 100 eigenvalues were used to compute the \textsc{NetLSD} embeddings for each graph. Note that this was not possible for the UK 2002 graph due to its large size. Observe that we can process graphs with millions of edges with reasonably low approximation error. UK 2002, a graph with $\approx 260M$ edges, was processed under half an hour by all of our proposed models. We note that when $b = 500000$, \textsc{gabe} and \textsc{santa} take a significant amount of time to compute on the Stanford and Flickr graphs due to their dense nature. Thus, we posit that one must consider the graph's density when setting the value of $b$.

\section{Conclusion}\label{sec:conclusion}

This paper proposes three graph descriptors and streaming algorithms with constant space complexity to construct them. Our descriptors extend the state-of-the-art graph descriptors and approximate their embeddings over graph streams. Experiments show that while using very less memory, our descriptors provide results comparable to SOTA descriptors, which store the entire graph in memory. We demonstrate the scalability of our algorithms to graphs with millions of edges (which is not possible for existing methods). We hope to introduce descriptors for attributed graphs that meet our constraints in the future. Another interesting future direction is to explore neural networks that can process edge streams, combining the scalability of stream-based methods and the classification prowess of graph convolutional networks.

\bibliographystyle{ACM-Reference-Format}
\bibliography{main}

\end{document}